\documentclass{article}

\PassOptionsToPackage{numbers,compress}{natbib}
\usepackage[main,final]{neurips_2026}

\usepackage[utf8]{inputenc}
\usepackage[T1]{fontenc}
\usepackage{url}
\usepackage{booktabs}
\usepackage{amsfonts}
\usepackage{nicefrac}
\usepackage{microtype}
\usepackage{xcolor}
\usepackage{etoc}
\usepackage{amsmath,amssymb,amsthm,mathtools,bm}
\usepackage{graphicx}
\usepackage{subcaption}
\usepackage{tabularx}
\usepackage{multirow}
\usepackage{pifont}
\usepackage{enumitem}
\usepackage{listings}
\usepackage{cleveref}
\usepackage{wrapfig}
\usepackage{algorithm}
\usepackage{algpseudocode}
\usepackage{titletoc}
\usepackage{placeins}

\theoremstyle{plain}
\newtheorem{theorem}{Theorem}[section]

\theoremstyle{definition}
\newtheorem{definition}[theorem]{Definition}

\theoremstyle{remark}

\title{Mitigating Memorization Where It Happens}

\author{%
  Xuanqi Zhang$^{1,2}$\thanks{Equal contribution.}
  \quad
  Haoyang Shang$^{1,2}$\footnotemark[1]
  \quad
  Xiaoxiao Li$^{1,2}$\thanks{Corresponding author.}
  \\
  $^{1}$Department of Electrical and Computer Engineering,\\
  The University of British Columbia\\
  $^{2}$Vector Institute
}

\begin{document}

\maketitle

\begin{abstract}
Large language models (LLMs) can memorize and reproduce training sequences verbatim, raising privacy and copyright concerns.
Existing mitigation methods apply interventions uniformly, often degrading predictions at the many positions unrelated to verbatim recall.
We study a token-level log-likelihood ratio between a target model and a domain-competent reference with lower exposure to the evaluated sequences.
Exposure-controlled experiments support its use for localizing positions associated with sequence-level verbatim recall.
Across evaluated settings, a minority of positions carries a disproportionate share of the positive signal, motivating selective activation correction.
From a first-order surrogate of the memorization-reduction objective, we derive a \textit{probe--steer framework}, which decomposes intervention into a probe that flags high-signal activations and a steer that applies a correction only when the probe exceeds a threshold. We propose Gated Subspace Steering (GSS) as its practical instantiation: under a background normalization constraint, the probe--steer pairs that are optimal for this surrogate follow in closed form from the leading singular vectors of a whitened gradient-weighted activation matrix. Across four benchmarks, GSS achieves memorization suppression comparable to strong baselines, with favorable utility trade-offs in several settings, while requiring 10--100$\times$ less compute than optimization-based alternatives.
Code is available at \url{https://github.com/ubc-tea/Mitigating-Memorization-Where-It-Happens}.
\end{abstract}

\section{Introduction}
%Background
Large language models (LLMs) can reproduce training sequences verbatim, creating risks of training-data extraction and disclosure of sensitive or copyrighted content~\cite{carlini2021extracting,shi2024large,nasr2025scalable,xiong2025landscape,ni2025controllable}.
Downstream adaptation is a standard use of pretrained language models~\cite{howard2018universal}.
Additional training on downstream data can increase such reproduction, particularly when sequences are repeated or heavily represented~\cite{wen2023unveiling,lee2022deduplicating}.
Verbatim memorization can coexist with useful generalization: a model may perform well on unseen examples while retaining and reproducing particular training sequences.
Our goal is to reduce this sequence-level reproduction while preserving the model's utility on other inputs.
Motivated by the pitfalls of evaluating memorization~\cite{bayat2024pitfalls}, we treat memorization suppression and generalization performance as separate evaluation dimensions, and study how inference-time interventions can control their trade-off without updating model weights.

%Existing Methods
Existing approaches offer several ways to modify model behavior.
\emph{Machine unlearning} aims to remove the influence of designated training examples~\cite{cao2015towards,bourtoule2021machine,graves2021amnesiac,wang2025efficient}, with methods such as NPO optimizing an alignment-inspired objective to mitigate the instability of gradient ascent~\cite{zhang2024negative}.
\emph{Knowledge editing} focuses on changing specific factual associations while preserving unrelated behavior~\cite{wang2024easyedit,feng2025geoedit}.
For example, SERAC uses an external memory and a routing mechanism to implement targeted edits~\cite{mitchell2022memory}.
\emph{Memorization mitigation} directly evaluates whether an intervention reduces the reproduction of training sequences~\cite{chang2024localization,suri2025mitigating,chen2024multi,gupta2025pruning}.
Localization-based approaches identify and suppress neurons or parameters associated with such reproduction.
These objectives overlap, and unlearning or editing methods can also be evaluated for memorization suppression.
We include them as objective-related baselines and separately compare activation-space interventions to assess the role of selective steering.

%Steering
Activation steering provides a lightweight mechanism for controlling
model behavior by modifying hidden states during inference, without
updating model parameters~\cite{marks2023geometry,rimsky2024steering,hegazy2025guiding}.
A common approach adds a fixed direction at every token position.
For memorization mitigation, such uniform intervention may perturb
useful predictions even where little correction is needed.
We investigate selective intervention using the token-wise
log-likelihood difference between the target model and a
domain-competent reference with lower exposure to the evaluated sequences.
Across our evaluated settings, a minority of positions accounts for
a disproportionate share of the positive likelihood difference.
Exposure-controlled analyses further associate this signal with
sequence-specific recall, while showing that an individual high-signal
position can also reflect legitimate learning.
These findings motivate using the signal to guide selective
activation correction, with its effectiveness assessed through
sequence-level reproduction and independent utility metrics.

These observations motivate separating the decision of where to
intervene from the choice of how to modify an activation.
A probe can identify candidate intervention positions from the
current hidden state, while a distinct steering direction specifies
the correction.
The target--reference likelihood signal guides offline calibration;
during generation, intervention depends only on the learned
activation probe and does not require evaluating the reference model.
We consider both a single-position gate and a variant requiring
threshold exceedance at two consecutive positions, and evaluate
their effects on sequence-level reproduction and utility separately.

Building on this design, we propose \textbf{Gated Subspace Steering
(GSS)}, an inference-time intervention combining a low-rank
\emph{probe}, a threshold \emph{gate}, and a corrective \emph{steer}.
During calibration, the target--reference signal identifies
high-signal positions whose activations and loss gradients guide
the construction of probe--steer directions.
A covariance-weighted linear surrogate yields a closed-form
direction-finding procedure based on singular value decomposition.
During generation, the gate selectively applies activation
corrections using the calibrated directions.
The theoretical analysis characterizes the local surrogate;
sequence-level suppression and utility under the deployed
intervention are evaluated empirically.
Our contributions are as follows:

\begin{itemize}
    \item We study the token-wise decomposition of target--reference
    excess log-likelihood, which serves as a signal for localizing positions
    associated with sequence-level verbatim recall.
    Exposure-controlled experiments validate its relationship to
    training exposure, and measurements across twelve model--dataset
    settings characterize the concentration of its positive mass.

    \item We introduce GSS, which separates activation-based probing
    from corrective steering and applies corrections through a
    configurable gate.
    We derive closed-form directions for a constrained first-order
    linear surrogate and distinguish this result from the behavior
    of the nonlinear intervention during generation.

    \item We evaluate memorization suppression and utility separately
    across pretrained and fine-tuned models.
    GSS achieves suppression comparable to strong baselines, with
    favorable utility trade-offs in several settings and low
    inference-time overhead.
    We further examine how gate configuration affects suppression
    and incidental changes to useful outputs.
\end{itemize}

\section{Related work}
\subsection{Memorization in language models}

Prior work shows that LLMs can memorize rare training examples and reproduce them verbatim at inference time~\cite{carlini2019secretsharer,dang2025memories,zhang2025beyond, ippolito2022preventing}. 
Carlini et al.~\cite{carlini2019secretsharer, carlini2022privacy} propose the exposure metric to quantify how much more likely a model is to generate a specific sequence than expected under a random baseline. 
Subsequent work extends this analysis across architectures and scales~\cite{carlini2022quantifying,ruzzetti2025private,huang2024demystifying,zhang2025get}, identifying training-data duplication, model size, and training duration as key drivers of memorization. 
Memorization also varies across examples and evaluation settings~\cite{barbulescu2024each,cheng2025approaching}.
In parallel, black-box extraction attacks~\cite{carlini2021extracting} demonstrate that memorization can lead to the recovery of verbatim training data from models such as GPT-2~\cite{radford2019language}.
Beyond verbatim reproduction, Zeng et al.~\cite{zeng2023exploring} study paraphrase and idea memorization.
Our work primarily targets sequence-level verbatim reproduction

\subsection{Activation steering and representation engineering}

Activation steering modifies internal representations at inference
time to control model behavior.
A common approach is contrastive activation addition, which extracts
a direction from differences between contrasting examples and applies
an update of the form $h \leftarrow h + \alpha v$
~\cite{rimsky2024steering,marks2023geometry,turner2023steering,stoehr2024activation}.
Related methods use representation directions to encode semantic
attributes or task signals, while adaptive and conditional approaches
vary interventions according to the input or activation state
~\cite{leeprogramming,wang2025steering,hegazy2025guiding,zhao2025adasteer,liu2026cobra}.
Benchmarking work further evaluates representation-based interventions across behaviors~\cite{wu2025axbench}.
Two design choices are relevant to our setting: how the correction
direction is obtained, and when the correction is applied.
GSS derives separate probe and steer directions from an
activation--gradient coupling under a covariance constraint,
then uses a calibrated gate to control their application.
We compare activation-space interventions because they share this
mechanism of behavioral control, and evaluate their suitability
for reducing sequence-level verbatim reproduction while preserving
utility.
\begin{figure}[b]
    \centering
    \begin{subfigure}[t]{0.45\textwidth}
        \centering
        \includegraphics[width=\linewidth]{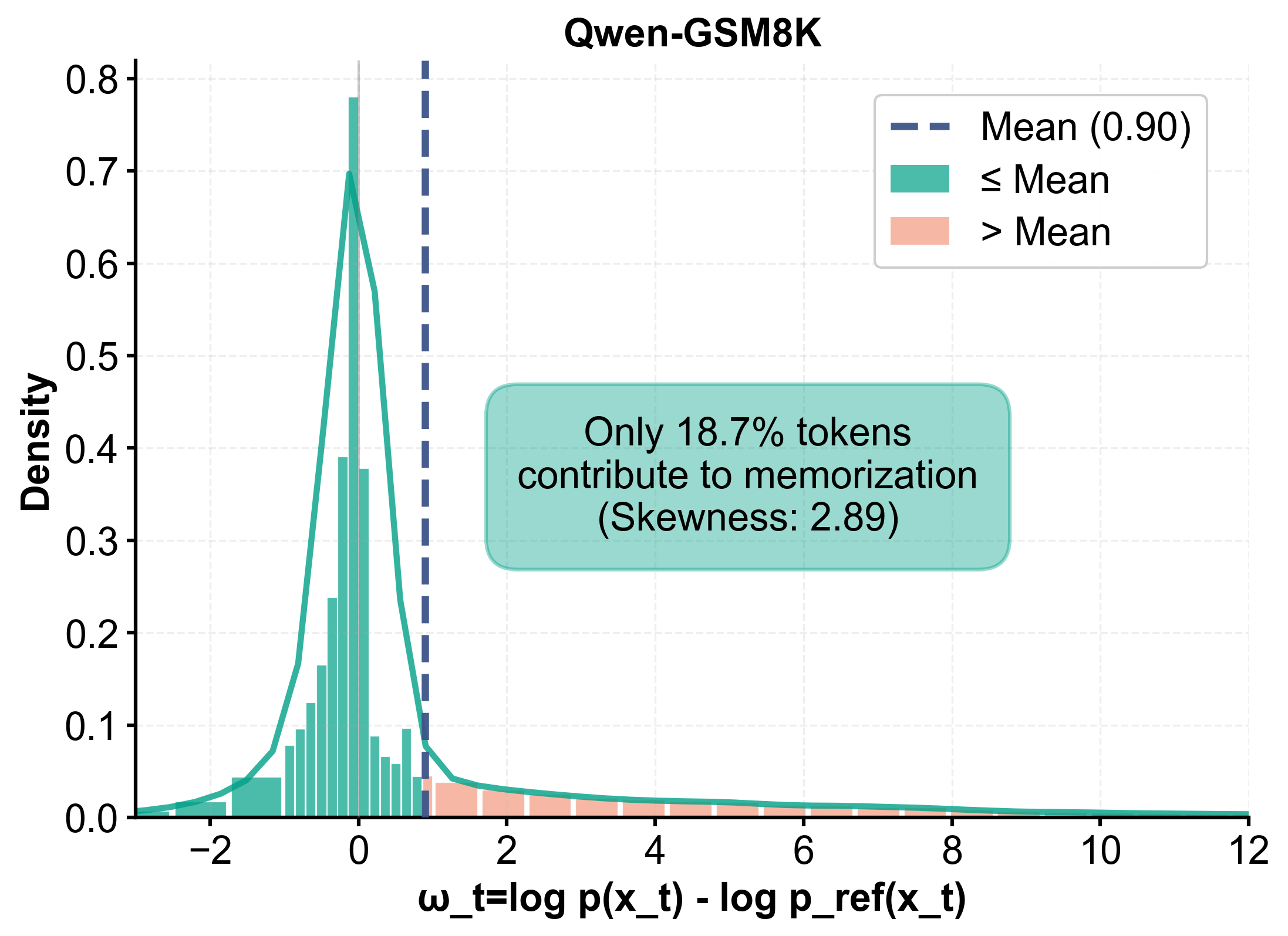}
        \caption{Distribution of the token-level signal.}
        \label{fig:teaser_distribution}
    \end{subfigure}\hfill
    \begin{subfigure}[t]{0.45\textwidth}
        \centering
        \includegraphics[width=\linewidth]{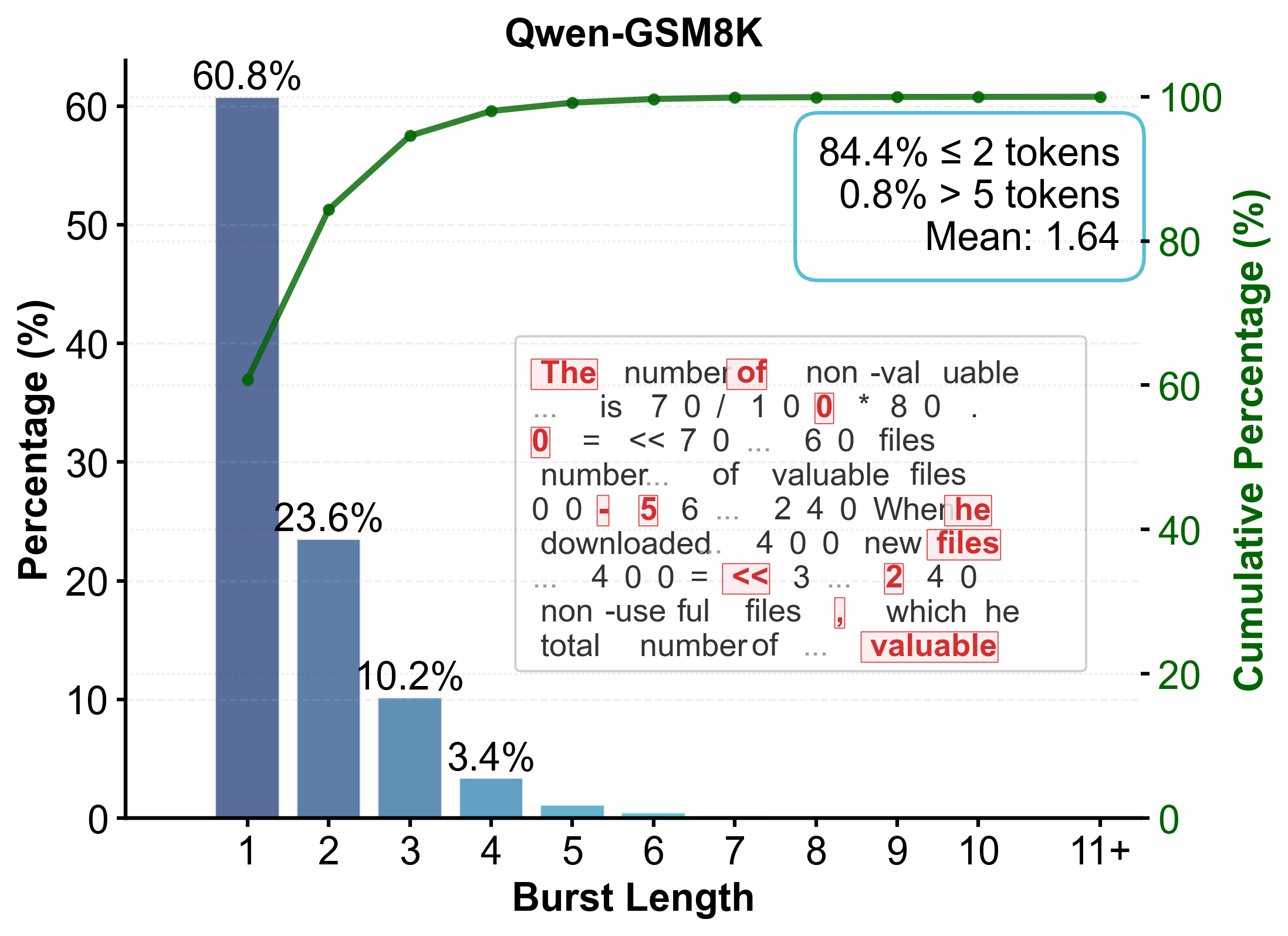}
        \caption{Lengths of positive-signal runs.}
        \label{fig:teaser_burst}
    \end{subfigure}
    \caption{
    \textbf{Token-level excess likelihood on Qwen3-GSM8K.}
    (a) The distribution of $\omega_t$ is right-skewed
    (skewness $2.89$), with some positions exhibiting much
    larger target--reference likelihood differences than others.
    (b) Consecutive runs of positions with $\omega_t>0$
    have a mean length of $1.64$; $60.8\%$ of runs contain
    one position.
    These statistics describe the likelihood signal, rather
    than the lengths of memorized spans or ground-truth
    transitions between recall and generalization.
    Cross-setting concentration measurements are reported in
    Appendix~\ref{app:additional_pilot_results}.
    }
    \label{fig:teaser}
\end{figure}

\section{Localizing excess predictive advantage within sequences}
\label{sec:token_level}

We evaluate memorization through sequence-level verbatim reproduction,
using the standard notion of $(n,k)$-memorization:

\begin{definition}[$(n,k)$-Memorization]
\label{def:nk_memorization}
Let $s = (x_1, \dots, x_n)$ be an $n$-token sequence that appears in
the training set of a language model $M$.
We say that $s$ is \emph{$(n,k)$-memorized} by $M$ if, when prompted
with the prefix $(x_1, \dots, x_k)$, the model reproduces the remaining
tokens $(x_{k+1}, \dots, x_n)$ exactly under greedy decoding.
\end{definition}

This definition specifies the behavior we seek to mitigate.
To guide selective intervention, we additionally examine how the
target model's predictive advantage over a reference is distributed
within a sequence.
The reference should capture broadly reusable linguistic and domain
regularities while having lower exposure to the evaluated sequences.
It need not be more accurate than the target or be its immediately
preceding checkpoint.

For a sequence $x_{1:T}$, define the excess log-likelihood
\begin{equation}
\label{eq:sequence_excess}
S(x_{1:T})
=
\log p_\theta(x_{1:T})
-
\log p_{\mathrm{ref}}(x_{1:T})
=
\sum_{t=1}^{T}\omega_t.
\end{equation}
This signed quantity measures the target's model-based code-length
advantage over the reference, connecting to information-theoretic treatments of memorization~\cite{morris2025much}.
Its token-wise decomposition identifies where that advantage arises;
it does not assign ground-truth memorization labels to individual
tokens.

\begin{definition}[Token-Level Excess Log-Likelihood Signal]
\label{def:memorization_signal}
Given a sequence $x_{1:T}$, the \emph{token-level excess log-likelihood signal} at position $t$ is
\begin{equation}
\label{eq:omega}
\omega_t
\;\triangleq\;
\log p_\theta(x_t \mid x_{<t}) -\log p_{\mathrm{ref}}(x_t \mid x_{<t}),
\end{equation}
where $p_\theta$ and $p_{\mathrm{ref}}$ are the target and reference models, respectively.
\end{definition}

The reference helps distinguish exposure-related gains from
predictability shared by both models.
For descriptive analysis, we define the positive signal mass as
\begin{equation}
\label{eq:positive_signal_mass}
m_t = \max(0,\omega_t),
\qquad
\mathcal{M}_{+}(x) = \sum_{t=1}^{T} m_t.
\end{equation}
Unlike the signed sequence-level gap $S(x)$,
$\mathcal{M}_{+}(x)$ retains positive contributions without
cancellation by negative ones.
For example, suppose the training set contains
``\textit{Alice's phone number is 555-0142}''
and the target reproduces it.
If both models predict ``is'' after ``phone number'' with similar
high probability, then $\omega_t \approx 0$: this position
contributes little to $\mathcal{M}_{+}(x)$ despite the target's
high confidence.
If the target has learned the particular number through exposure
while the reference has not, the number's tokens can receive
substantially higher probability under the target and contribute
strongly to $\mathcal{M}_{+}(x)$.
Target likelihood alone does not distinguish these two sources
of predictability; the reference provides a comparative baseline.
We use \emph{high-signal} and \emph{low-signal positions} to describe
the resulting partition.

\textbf{Concentration of positive signal mass.}
Figure~\ref{fig:teaser} illustrates the distribution of the
token-level signal and the lengths of its positive-signal runs
on Qwen3-GSM8K.
Across twelve model--dataset settings comprising 11,833 sequences
and 1.72 million scored positions, the top 10\% of positions account
for 30\%--89\% of the total positive signal mass.
The Gini coefficient ranges from 0.50 to 0.91.
This concentration does not imply that positive-signal positions
are rare: 47\%--94\% of positions have $\omega_t>0$, but many
contribute only small positive values.
We therefore characterize the observation as concentration of
positive excess likelihood, without assuming a particular
heavy-tailed distribution.
These measurements motivate selective intervention.

\textbf{Relationship to training exposure.}
We validate the signal using controlled comparisons in which
sequence inclusion determines exposure.
On seed-controlled canary pairs, $\omega_t$ predicts positions
with top-decile inclusion gain with an AUROC of 0.99.
In a cross-model evaluation, scores from a Pythia-2.8B pair
predict independently measured Pythia-6.9B inclusion-gain labels
with an AUROC of 0.95, compared with 0.48 for target confidence.
On natural Pile text, positive-signal positions are
7.7 times as likely as the remaining positions to occur within
a 20-token verbatim reproduction.
These results support exposure-sensitive localization
(Appendix~\ref{app:token_signal_validation}).

\textbf{Implications for intervention.}
An isolated high-signal position can reflect legitimate learning,
so its suppression need not improve sequence-level memorization.
Moreover, the offline likelihood signal and the activation probe
used during generation are distinct quantities.
We therefore evaluate the deployed gate directly, including
whether requiring threshold exceedance at consecutive positions
reduces incidental intervention.
This motivates GSS as a selective activation intervention whose
benefits and costs are measured through sequence-level
reproduction and independent utility metrics.

\section{Gated subspace steering}
\label{sec:gss}
\label{sec:method}

\begin{figure}[t]
  \centering
  \includegraphics[width=\linewidth]{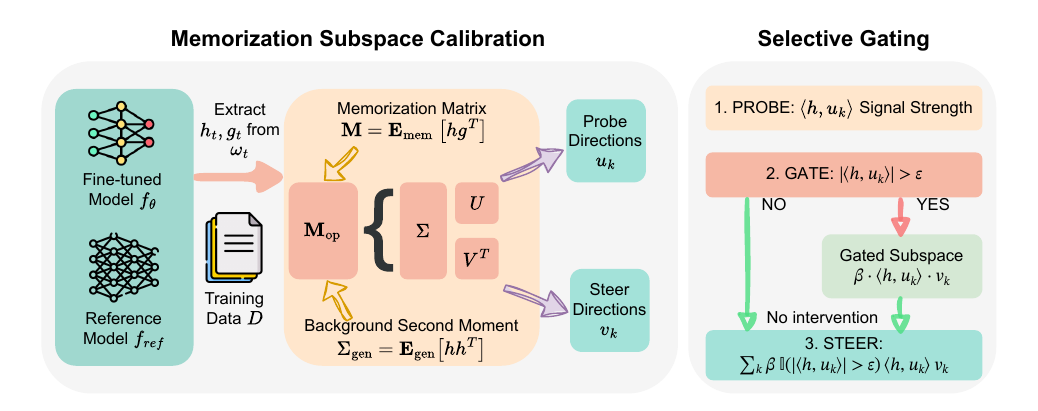}
  \caption{\textbf{Gated Subspace Steering overview.} \textbf{(a) Direction estimation.} The target--reference signal partitions positions into high- and low-signal subsets. Activation--gradient statistics from high-signal positions and a second moment from low-signal positions determine probe ($u_k$) and steer ($v_k$) directions through whitening and SVD. \textbf{(b) Intervention.} During generation, the gate applies a correction when the magnitude of a probe response exceeds its threshold. The response controls the signed correction along the steer. Gate decisions are not ground-truth memorization labels.}
  \label{fig:gss_overview}
\end{figure}

Section~\ref{sec:token_level} shows that positive excess likelihood is concentrated in a minority of positions. Gated Subspace Steering (GSS) exploits this by intervening only where needed. It acts on the hidden state $h_t \in \mathbb{R}^d$ at a single chosen layer and has two stages (Fig.~\ref{fig:gss_overview}). An offline stage (Fig.~\ref{fig:gss_overview}a) estimates $K$ pairs of directions: a \emph{probe} $u_k$ that decides \emph{where} to intervene and a \emph{steer} $v_k$ that decides \emph{how}. An online stage (Fig.~\ref{fig:gss_overview}b) applies a correction along $v_k$ only when the response to $u_k$ exceeds a threshold. We formulate the direction-finding problem in Section~\ref{sec:gss_formulation}, solve it in closed form in Section~\ref{sec:gss_optimization}, and then describe the gated intervention used during generation.

\subsection{Problem formulation}
\label{sec:gss_formulation}

We run the target and reference models on a set of training sequences $\mathcal{D}$ and compute $\omega_t$ (Eq.~\ref{eq:omega}) at every position. This splits positions into a high-signal set $\mathcal{H} = \{t : \omega_t > 0\}$ and a low-signal background set $\mathcal{L} = \{t : \omega_t \le 0\}$ (left of Fig.~\ref{fig:gss_overview}a). 
During generation, neither the reference model nor gradients are available, so the correction must be computed from $h_t$ alone. We use a correction that reads the hidden state along a probe $u$ and writes along a steer $v$:
\begin{equation}
\label{eq:correction}
\Delta h_t = -\beta\,(u^\top h_t)\,v, \qquad \beta > 0 .
\end{equation}
We keep $u$ and $v$ separate because recognizing a high-signal state and changing the next-token prediction are different tasks: the first depends on how activations are distributed, the second on how the output responds to them, and the two directions need not coincide.

\textbf{Objective.}
The correction should lower the target's likelihood on high-signal positions. Let
$\ell_{\mathcal{H}} = \frac{1}{|\mathcal{H}|}\sum_{t\in\mathcal{H}} \log p_\theta(x_t \mid x_{<t})$
be the mean log-likelihood on $\mathcal{H}$, and let $g_t = \nabla_{h_t}\ell_{\mathcal{H}}$ be its gradient with respect to the hidden state.\footnote{$g_t$ is the negative of the cross-entropy gradient. Batching and gradient normalization are described in Appendix~\ref{app:calibration_gradients}.} To first order, applying Eq.~\eqref{eq:correction} at every $t \in \mathcal{H}$ changes $\ell_{\mathcal{H}}$ by
\begin{equation}
\label{eq:first_order}
\Delta \ell_{\mathcal{H}} \approx \sum_{t\in\mathcal{H}} g_t^\top \Delta h_t
= -\beta \sum_{t\in\mathcal{H}} (u^\top h_t)(g_t^\top v)
= -\beta\, u^\top M_{\mathrm{hi}}\, v,
\qquad
M_{\mathrm{hi}} = \sum_{t\in\mathcal{H}} h_t g_t^\top .
\end{equation}
Maximizing $u^\top M_{\mathrm{hi}} v$ therefore maximizes the first-order likelihood reduction on high-signal positions. The \emph{memorization matrix} $M_{\mathrm{hi}}$ (Fig.~\ref{fig:gss_overview}a, top) is large along pairs $(u, v)$ for which the probe responds strongly exactly where moving along the steer lowers the likelihood.

\textbf{Background constraint.}
Maximizing $u^\top M_{\mathrm{hi}} v$ alone favors probes with large responses everywhere, including on background positions, where the gate would then fire. We prevent this by fixing the scale of the probe response on $\mathcal{L}$ using the background second moment (Fig.~\ref{fig:gss_overview}a, bottom):
\begin{equation}
\label{eq:background}
\Sigma_{\mathrm{lo}} = \frac{1}{|\mathcal{L}|}\sum_{t\in\mathcal{L}} h_t h_t^\top,
\qquad
C = \Sigma_{\mathrm{lo}} + \gamma I ,
\end{equation}
with a small $\gamma > 0$ for numerical stability. Since $u^\top \Sigma_{\mathrm{lo}} u$ is the mean squared probe response on background positions, requiring $u^\top C u = 1$ makes response scales comparable across probes, so a single threshold $\epsilon$ can be read in units of the typical background response.

Stacking $U = [u_1,\dots,u_K]$ and $V = [v_1,\dots,v_K]$, we seek $K$ probe--steer pairs that solve
\begin{equation}
\label{eq:spectral}
\max_{U,V}\ \operatorname{tr}\!\left(U^\top M_{\mathrm{hi}} V\right)
\quad \text{s.t.}\quad U^\top C U = I_K,\ \ V^\top V = I_K ,
\end{equation}
where the first constraint bounds probe responses on the background and the second keeps the $K$ steers orthonormal, preventing them from collapsing onto the same component.

\subsection{Optimization}
\label{sec:gss_optimization}

Although Eq.~\eqref{eq:spectral} is non-convex, it has an exact closed-form solution, connecting to spectral representation methods~\cite{belrose2023leace,qiu2024spectral}.

\begin{theorem}[Optimal whitened directions]
\label{thm:spectral}
\label{thm:optimal_solution}
Let $C = LL^\top$ be the Cholesky decomposition of $C$ and let $L^{-1}M_{\mathrm{hi}} = PSQ^\top$ be an SVD with singular values in descending order. Then $U^\ast = L^{-\top}P_{:,1:K}$ and $V^\ast = Q_{:,1:K}$ solve Eq.~\eqref{eq:spectral}, with optimal value $\sum_{k=1}^{K}\sigma_k(L^{-1}M_{\mathrm{hi}})$.
\end{theorem}

\emph{Proof sketch.} Substituting $\tilde U = L^\top U$ turns the first constraint into $\tilde U^\top \tilde U = I_K$ and the objective into $\operatorname{tr}(\tilde U^\top L^{-1} M_{\mathrm{hi}} V)$. By the von Neumann trace inequality, this is maximized by the leading $K$ left and right singular vectors of $L^{-1}M_{\mathrm{hi}}$; mapping back gives $U^\ast$. The full proof is in Appendix~\ref{app:derivation}. \hfill$\square$

The solution corresponds to the middle of Fig.~\ref{fig:gss_overview}a and takes three steps: (i) whiten the memorization matrix by the background second moment, $L^{-1}M_{\mathrm{hi}}$; (ii) compute its SVD; (iii) take the top-$K$ left singular vectors, mapped back by $L^{-\top}$, as probes and the top-$K$ right singular vectors as steers. The procedure requires no iterative optimization and no update to the model weights; its one-time cost is dominated by the SVD of a $d \times d$ matrix.

\subsection{Gated intervention}
\label{sec:gss_inference}

During generation, GSS keeps $(U^\ast, V^\ast)$ fixed and performs the three steps of Fig.~\ref{fig:gss_overview}b at each position:
\begin{equation}
\label{eq:gss}
\underbrace{s_{t,k} = u_k^\top h_t}_{\text{1. probe}},
\qquad
h_t' = h_t - \beta \sum_{k=1}^{K}
\underbrace{\mathbb{I}\!\left(|s_{t,k}| > \epsilon\right)}_{\text{2. gate}}\,
\underbrace{s_{t,k}\, v_k}_{\text{3. steer}} .
\end{equation}
The \emph{probe} measures how strongly $h_t$ expresses mode $k$. The \emph{gate} activates mode $k$ only when the magnitude of this response exceeds $\epsilon$. The \emph{steer} subtracts a correction along $v_k$ scaled by the signed response. 
A run requirement $r$ optionally restricts the gate to persistent responses: with $c_t = c_{t-1}+1$ if the gate of the primary mode is active at position $t$ and $c_t = 0$ otherwise, every correction term at $t$ is multiplied by $\mathbb{I}(c_t \ge r)$. The single-position gate is $r=1$; $r=2$ intervenes from the second consecutive exceedance onward.

\begin{figure}[t]
  \centering
  \includegraphics[width=0.9\linewidth,trim={18pt 232pt 18pt 0pt},clip]{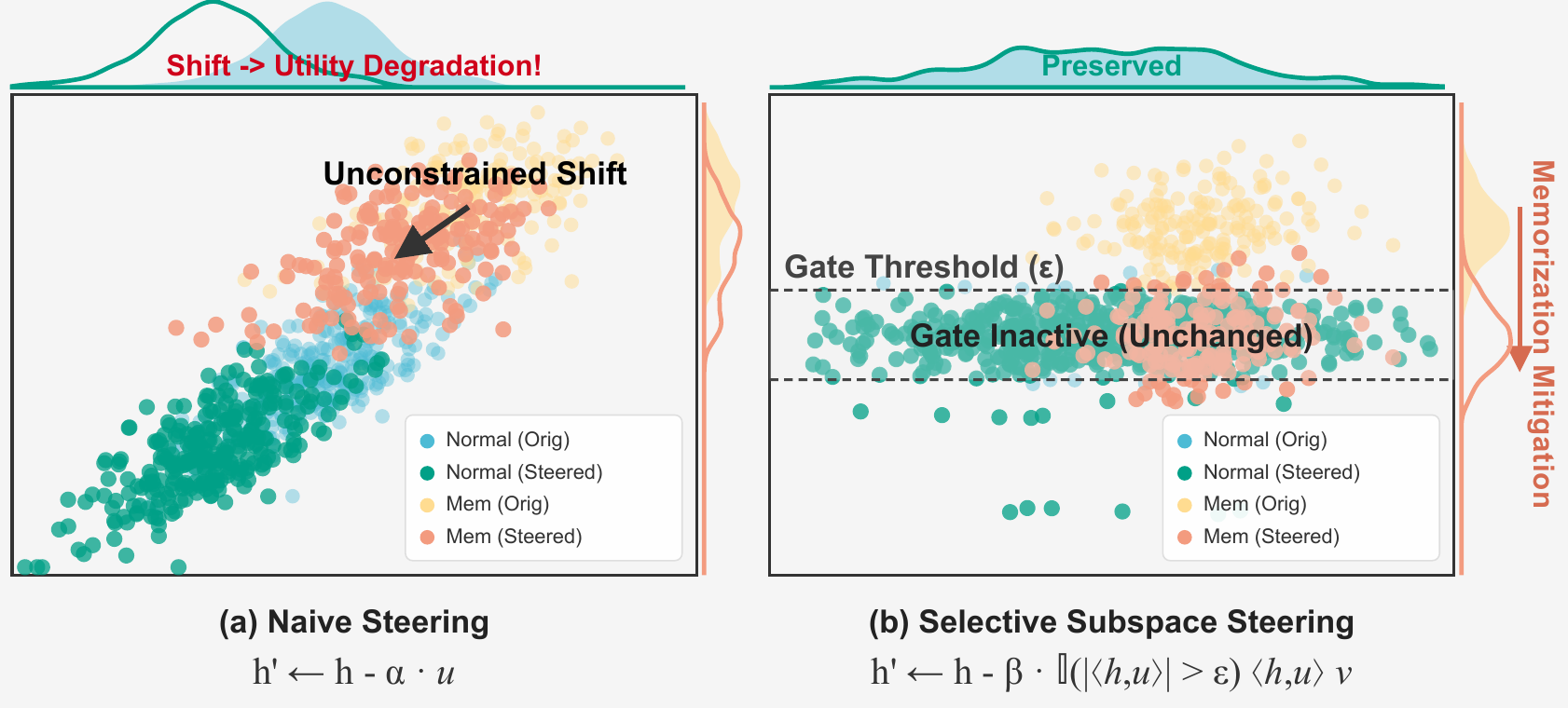}
  \caption{\textbf{Fixed and gated activation interventions.} (a) A fixed steering update is applied at every position. (b) GSS applies its correction only where a probe response exceeds the threshold. The illustrated regions are schematic and do not establish a separation between memorized and generalized computation.}
  \label{fig:gss_gating}
\end{figure}

Fig.~\ref{fig:gss_gating} contrasts this with fixed steering, which applies the same update at every position (Fig.~\ref{fig:gss_gating}a). A fixed update moves all activations, including background ones that need no correction, and shifts the model's overall behavior. GSS (Fig.~\ref{fig:gss_gating}b) applies the correction only above the threshold and leaves the remaining positions untouched.
We set $\epsilon$ to the $(1-\eta)$ quantile of probe responses on text, so $\eta$ is the nominal fraction of positions at which a probe fires (Appendix~\ref{app:epsilon}). At inference, GSS needs neither the reference model nor gradients, adding only $O(Kd)$ inner products per token at the intervened layer.
Theorem~\ref{thm:spectral} characterizes the first-order, ungated surrogate of Eq.~\eqref{eq:spectral}. A large probe response marks a candidate position.

\textbf{Contrastive configurations.}
Eq.~\eqref{eq:gss} and Theorem~\ref{thm:spectral} cover the whitened configuration (W).
The contrastive configurations take the leading singular pair of the difference between the high- and low-signal activation--gradient matrices without whitening, accumulated at every position (C) or only at the masked positions (C-M; Appendix~\ref{app:calibration_gradients}), and apply the one-sided update $h_t' = h_t - \beta\,[u^\top h_t-\epsilon]_+\,v/\max(u^\top v,10^{-3})$ (Appendix~\ref{app:alpha_derivation}).

\section{Experiments}
\label{sec:experiments}

\subsection{Experimental setup}

We evaluate GSS across four settings spanning memorization-mitigation
benchmarks, scalability tests, and realistic fine-tuning scenarios:
(i)~\textbf{TinyMem}~\citep{suri2025mitigating}, a benchmark suite
designed specifically for memorization mitigation;
(ii)~\textbf{Pythia-2.8B and Pythia-6.9B} on
their pretraining corpus, to assess scalability;
(iii)~\textbf{Llama-3.2-1B and Qwen3-0.6B fine-tuned on GSM8K}
(reasoning)~\citep{xie2025memorization} and
\textbf{UltraChat} (dialogue)~\citep{ding2023enhancing}.
We apply GSS at layer 4 for TinyMem models and layer 18 for Pythia models (Layer Selection Ablation is provided in Appendix~\ref{app:layer_ablation}).
\textbf{Gate configurations.} Except in Section~\ref{sec:consecutive_main}, all results in this section use the single-position gate ($r=1$). Table~\ref{tab:4layer_combined} uses whitened directions and Table~\ref{tab:pythia_unlearning} contrastive directions (Appendix~\ref{app:alpha_derivation}), both with rank one. Table~\ref{tab:gate_configs} lists the configuration of every result that reports $r=2$.

\textbf{Baselines.}
We organize baselines by intervention target.
\emph{Parameter-space} methods modify model weights $\theta$ and
include three families: \emph{unlearning} (GA~\citep{jang2023knowledge}, NPO~\citep{zhang2024negative}) and
localization-based (Greedy, SOU, Dura, Subnet,
BalSub~\citep{sakarvadia2024mitigating}), plus the LoRA-based
O$^3$~\citep{gao2025large}; \emph{model editing}
(ROME~\citep{meng2022locating}, MEMIT~\citep{meng2022mass}); and
\emph{task arithmetic} (TaskVector~\citep{ilharco2022editing}).
\emph{Inference-time} methods operate on activations $h$
without weight updates: feature-suppression methods (Zero, Slimming,
Hard Concrete, IG~\citep{chang2024localization}) and steering
methods (LUNAR~\citep{shen2025lunar},
AlphaSteer~\citep{sheng2025alphasteer}, LEACE~\citep{belrose2023leace}), to which our method belongs.

\textbf{Metrics.}
We report \textbf{Memorization} via
$(n,k)$-memorization~\citep{suri2025mitigating,carlini2021extracting}, adversarial extraction and approximate memorization results are in Appendix \ref{app:adversarial},
\textbf{Utility} via Perplexity on held-out test sets and
downstream task Accuracy and \textbf{Efficiency} via wall-clock time
of the inference stage. 

\begin{table*}[t]
\centering
\caption{\textbf{TinyMem Models} \cite{sakarvadia2024mitigating}. Comparison of memorization mitigation strategies on multiplicative math and language models under noise and backdoor settings. The results are averaged with 3 seeds. Metrics include $(n,k)$-memorization, test accuracy / perplexity, and wall-clock time of the mitigation pipeline (seconds). GSS uses whitened directions and the single-position gate ($r=1$).}
\scriptsize    
\renewcommand{\arraystretch}{1.15}
\setlength{\tabcolsep}{3pt}

\begin{tabular}{l c ccc ccc ccc ccc}
\toprule
\multirow{3}{*}{Method}
& \multirow{3}{*}{Year}
& \multicolumn{6}{c}{\textbf{Math Model}}
& \multicolumn{6}{c}{\textbf{Language Model}} \\
\cmidrule(lr){3-8} \cmidrule(lr){9-14}

&
& \multicolumn{3}{c}{Noise}
& \multicolumn{3}{c}{Backdoor}
& \multicolumn{3}{c}{Noise}
& \multicolumn{3}{c}{Backdoor} \\
\cmidrule(lr){3-5} \cmidrule(lr){6-8}
\cmidrule(lr){9-11} \cmidrule(lr){12-14}

&
& \%Mem $\downarrow$ & Acc $\uparrow$ & Time $\downarrow$
& \%Mem $\downarrow$ & Acc $\uparrow$ & Time $\downarrow$
& \%Mem $\downarrow$ & Perp $\downarrow$ & Time $\downarrow$
& \%Mem $\downarrow$ & Perp $\downarrow$ & Time $\downarrow$ \\
\midrule

Baseline & --
& 34.55 & 97.00 & --
& 99.44 & 96.81 & --
& 17.60 & 57.73 & --
& 100.00 & 58.45 & -- \\

Spectral reg & 2024
& 0.17 & 96.84 & 12254
& 99.89 & 96.68 & 1131.5
& 0.00 & 57.92 & 10850.2
& 100.00 & 90.42 & 238.6 \\

Loss trunc & 2024
& 30.58 & 96.97 & 2391
& 99.67 & 99.96 & 347.7
& 0.00 & 57.92 & 16080
& 100.00 & 58.45 & 6005 \\

\midrule

Both FT & 2024
& 0.00 & 97.17 & 28.45
& 0.00 & 91.36 & 20.35
& 0.00 & 51.63 & 1680
& 0.00 & 59.19 & 1629 \\

Clean FT & 2024
& 0.00 & 59.66 & 2.84
& 0.00 & 70.71 & 2.91
& 0.00 & 70.41 & 6.35
& 0.00 & 77.19 & 6.36 \\

Extra FT & 2024
& 0.00 & 73.22 & 27.33
& 0.00 & 74.63 & 18.58
& 0.00 & 52.00 & 1670
& 0.00 & 63.60 & 1624 \\

\midrule

HC & 2024
& 0.00 & 74.97 & 0.24
& 0.00 & 75.02 & 0.26
& 0.00 & 57.15 & 0.25
& 100.0 & 70.94 & 0.69 \\

Slim & 2024
& 0.00 & 63.34 & 1.47
& 0.00 & 57.12 & 1.41
& 0.00 & 57.10 & 0.42
& 100.0 & 70.40 & 0.89 \\

Act & 2024
& 0.00 & 60.22 & 0.37
& 0.00 & 78.89 & 0.27
& 0.00 & 71.32 & 0.37
& 100.00 & 70.98 & 0.43 \\

IG & 2024
& 24.30 & 96.98 & 2657
& 46.67 & 54.19 & 1586
& 0.00 & 57.12 & 1647
& 0.00 & 102.7 & 1612 \\

Zero & 2024
& 22.15 & 96.86 & 11.91
& 99.44 & 73.44 & 12.10
& 0.00 & 57.12 & 16.47
& 100.0 & 70.44 & 63.05 \\

Greedy & 2024
& 12.73 & 96.19 & 17.62
& 46.67 & 54.19 & 18.46
& 0.00 & 102.7 & 1612
& 100.0 & 119.3 & 217.9 \\

SOU & 2024
& 0.81 & 87.02 & 313.4
& 78.89 & 76.11 & 449.8
& 26.00 & 57.28 & 573.7
& 100.0 & 70.74 & 585.1 \\

Dura & 2024
& 13.72 & 96.45 & 2.78
& 0.00 & 65.63 & 3.23
& 2.80 & 57.52 & 0.73
& 100.0 & 73.43 & 3.36 \\

Dura-agg & 2024
& 14.71 & 96.41 & 2.60
& 24.69 & 86.18 & 3.08
& 0.40 & 68.47 & 0.75
& 100.0 & 85.14 & 3.33 \\

Subnet & 2024
& 20.33 & 96.91 & 0.48
& 0.00 & 28.66 & 0.46
& 14.80 & 57.28 & 0.41
& 94.87 & 70.83 & 1.00 \\

BalSub & 2025
& 0.66 & 91.97 & 7.29
& 0.00 & 92.01 & 6.96
& 0.80 & {57.14} & 299.5
& 0.00 & 71.86 & 917.9 \\

\midrule

TaskVector & 2022
& 0.00 & 89.31 & 0.37
& 9.44 & 93.20 & 0.37
& 16.80 & 57.71 & 0.45
& 100.00 & 70.41 & 0.67 \\

LEACE  & 2023
& 0.00 & 74.09 & 0.48
& 0.00 & 77.75 & 1.30
& 13.20 & 65.63& 0.69
& 100.00 & 65.12 & 1.70 \\

NPO & 2024
& 0.56 & 27.26 & 0.57
& 0.00 & 4.88 & 0.56
& 0.00 & 229.45 & 1.28
& 0.00 & 858.15 & 1.38 \\

\texttt{O$^3$} & 2025
& 5.12 & 95.37 & 1.83
& 66.67 & 61.23 & 2.32
& 56.52 & 13.60 & 0.49
& 73.09 & 70.00 & 1.14 \\

AlphaSteer & 2025
& 24.30 & 96.98 & 6.35
& 0.92 & 61.24 & 2.09
& 0.00 & 64.29 & 15.82
& 0.00 & 76.08 & 18.74 \\

Lunar & 2025
& 24.28 & 96.97 & 1.15
& 0.00 & 58.69 & 1.27
& 0.00 & 67.36 & 14.93
& 0.00 & 75.96 & 17.61 \\

\midrule

Ours & 2026
& {0.00} & {96.98} & {0.001}
& {0.00} & {96.82} & {0.001}
& {0.00} & 63.13 & {0.002}
& {0.00} & {63.17} & {0.003} \\

\bottomrule
\end{tabular}
\vspace{-0.5em}
\label{tab:4layer_combined}
\end{table*}

%We apply GSS at layer 4 for TinyMem models and layer 18 for Pythia models (Layer Selection Ablation is provided in Appendix~\ref{app:layer_ablation}).
%The gating threshold $\varepsilon$ is set to the 95th percentile of
%$|\langle h, u\rangle|$ on the validation set. The sensitivity of reference model is provided in Appendix~\ref{app:ref_sensitivity}.

\subsection{Experimental results on TinyMem}

As shown in Table~\ref{tab:4layer_combined}, our method yields zero observed reproduction under the evaluated protocol in all TinyMem settings. 
While some baselines reach zero memorization in isolated cases, they are often at the cost of catastrophic utility loss. 
Neuron- and weight-level pruning methods (HC, Slim, Act) also suppress memorization but substantially degrade accuracy, particularly on math tasks that require precise algorithmic structure. 
In contrast, our method preserves near-baseline accuracy on math models and stable perplexity on language models, demonstrating selective suppression without disrupting learned computation.
Backdoor settings further highlight the limitations of static interventions. 
Many baselines that perform well under noise fail completely under backdoor perturbations, often reverting to $100\%$ memorization (e.g., HC, Slim, Act, Zero, Greedy on language backdoor). 
In contrast, our method remains effective under backdoor conditions.
Moreover, our method introduces a negligible inference-time overhead, making it suitable as a lightweight deployment-time safeguard.

\subsection{Experimental results on Pythia}

On pretrained Pythia models, GSS reduces $(n,k)$-memorization from
$89.31\%$ to $6.96\%$ on Pythia-6.9B and from $52.87\%$ to $6.93\%$
on Pythia-2.8B, with test perplexity rising from $19.46$ to $29.15$
($+49.8\%$) and from $21.75$ to $28.26$ ($+29.9\%$), respectively.
This perplexity cost is larger than in fine-tuning settings because
pretrained models lack a clean forget/retain split, leaving less
margin for utility preservation. Among the evaluated methods, those matching GSS on perplexity
(HC, Subnet, BalSub) leave memorization above $84\%$, while those
approaching GSS on memorization (Dura, Dura-agg) inflate perplexity
by $20\%$--$77\%$ and require more wall-clock time. 

\begin{wraptable}{R}{0.52\textwidth}
\centering
\vspace{-1.1em}
\caption{Comparison on \textbf{Pythia} models across $(n,k)$-memorization, test perplexity, and time.}
\label{tab:pythia_unlearning}
\scriptsize
\setlength{\tabcolsep}{2.5pt}
\renewcommand{\arraystretch}{0.9}
\resizebox{0.50\textwidth}{!}{%
\begin{tabular}{l|ccc|ccc}
\toprule
 & \multicolumn{3}{c|}{\textbf{Pythia 2.8B}} 
 & \multicolumn{3}{c}{\textbf{Pythia 6.9B}} \\
\textbf{Method} 
& \%Mem $\downarrow$ & PPL $\downarrow$ & Time $\downarrow$
& \%Mem $\downarrow$ & PPL $\downarrow$ & Time $\downarrow$ \\
\midrule
Baseline
& 52.87 & 21.75 & -- 
& 89.31 & 19.46 & -- \\

HC 
& 44.55 & 21.35 & 2.25 
& 87.55 & 17.56 & 36.01 \\

Slim 
& 14.46 & 23.41 & 1.96
& 33.46 & 18.90 & 280.35 \\

Act 
& 14.46 & 24.43 & 13.01 
& 30.30 & 18.34 & 36.02 \\

Dura 
& 6.93 & 35.49 & 17.44 
& 14.65 & 23.34 & 44.34 \\

Dura-agg 
& 7.52 & 32.35 & 231.43 
& 10.98 & 34.40 & 320.60 \\

Subnet 
& 8.01 & 32.20 & 48.75 
& 84.95 & 17.16 & 519.44 \\

BalSub
& 5.94 & 31.52 & 2068.77 
& 86.73 & 17.15 & 233.42 \\
\midrule
\textbf{Ours} 
& \textbf{6.93} & \textbf{28.26} & \textbf{0.16}
& \textbf{6.96} & 29.15 & \textbf{0.21} \\

\bottomrule
\end{tabular}%
}
\end{wraptable}

\vspace{-0.1em}

\subsection{Pareto frontier analysis on modern LLMs}

\begin{figure}[t] % The * ensures it spans across both columns
\vspace{-1em}
    \centering
     \includegraphics[width=1\linewidth]{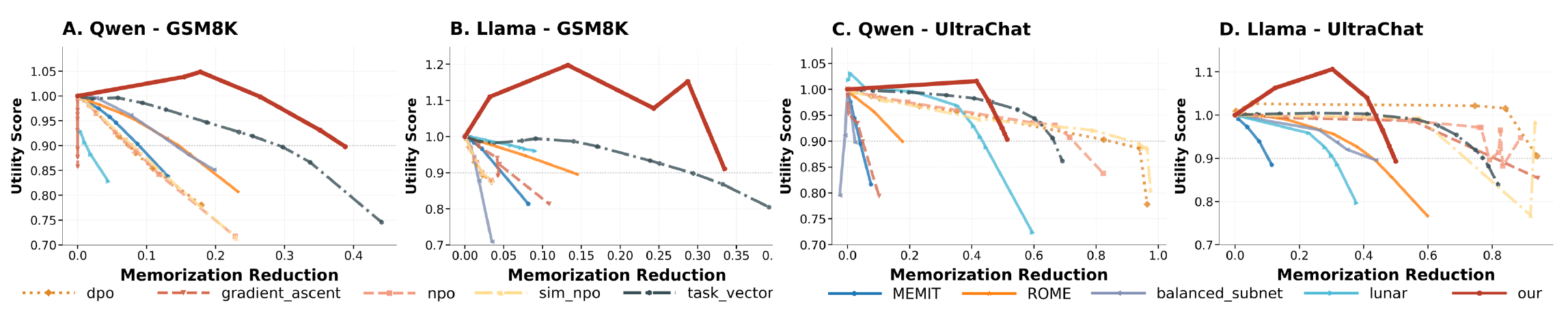}
    \caption{\textbf{Pareto Frontier Analysis.} 
    The plots visualize the trade-off between memorization reduction measured and downstream utility measured in log-likelihood. The ideal method occupies the top-right corner. 
    On reasoning tasks, GSS keeps the highest utility over the reduction range it reaches. 
    On instruction tasks,
    GSS keeps higher utility than the \textit{Editing} baselines over its range. }%, achieving significant memorization reduction with minimal impact on generation quality.}
    \vspace{-0.5em}
    \label{fig:pareto_frontier}
\end{figure}

To evaluate the efficiency of our intervention, we conduct a comprehensive Pareto analysis. 
Specifically, we sweep the \textit{learning rate} for Unlearning methods (e.g., GA), the \textit{regularization weight} for localization methods (e.g., BalSub), and the \textit{intervention magnitude} for Editing methods (e.g., ROME). 
Fig.~\ref{fig:pareto_frontier} shows the resulting trade-off curves. GSS covers a smaller range of memorization reduction than the strongest baselines and keeps high utility within that range. 

\textbf{Reasoning Tasks.} 
On complex reasoning benchmarks (first two in Fig.~\ref{fig:pareto_frontier}), we observe a clear hierarchy. 
\textit{Unlearning methods} (e.g., GA) struggle significantly. 
\textit{Knowledge Editing methods} (ROME, MEMIT) show better stability but hit a performance ceiling. 
GSS keeps the highest utility up to its largest reduction, with a slight \textbf{utility boost} (scores exceeding $1.0$) at moderate steering strengths, while task vectors reach larger reductions at lower utility. 
This suggests that for reasoning tasks, memorization acts as overfitting noise; by selectively projecting it out via a gated subspace, we recover latent generalization capabilities on the test set.

\textbf{Dialogue Tasks.}
For dialogue generation (the last two panels in Fig.~\ref{fig:pareto_frontier}), 
Within the range of memorization reduction attained by GSS, it maintains higher utility than the \textit{Knowledge Editing} baselines, with utility scores remaining above $1.0$ at moderate reductions.
Unlearning methods extend to larger reductions than GSS, although they have lower utility.
This comparison also highlights the distinction between preventing exact reproduction and preserving the quality of open-ended responses.

\begin{wrapfigure}{r}{0.4\columnwidth}
    \vspace{-1em}
    \centering
    \includegraphics[width=\linewidth]{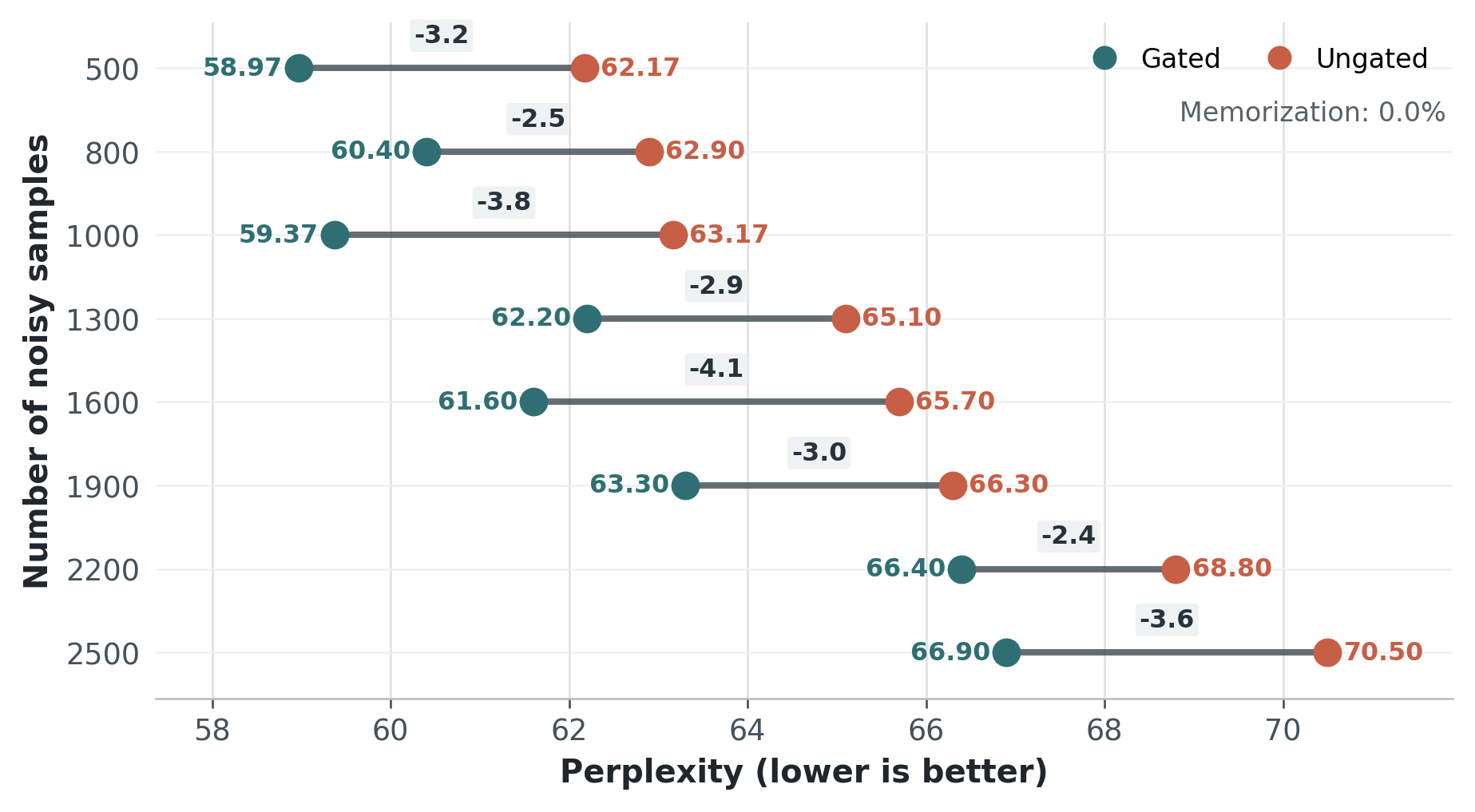}
    \vspace{-1.0em}
    \caption{
    \textbf{Gated vs. Ungated.} Gating exploits heavy-tailed sparsity.
    }
    \label{fig:gated_ungated_ppl}
    \vspace{-1em}
\end{wrapfigure}

\subsection{Gated vs.\ ungated intervention}
\label{sec:ablation_gating}

We isolate the contribution of the gating mechanism by comparing GSS
to its ungated variant: same calibrated $(u^{*}, v^{*})$, but the
indicator $\mathbb{I}(\cdot)$ is removed and the steer is applied
uniformly to all tokens. Fig.~\ref{fig:gated_ungated_ppl} shows the
comparison on TinyMem under varying noise levels: gated intervention
achieves $2.4$--$4.1$ PPL points lower than ungated, while both
maintain $0.0\%$ memorization.  Table~\ref{tab:gating_ablation} in
Appendix~\ref{app:gating_ablation} lists the exact values. Together, these results
isolate the gating mechanism as the source of GSS's utility-preservation. By restricting corrections to positions with high probe responses, the gate limits unnecessary intervention at other positions while retaining substantial suppression of verbatim reproduction.

\subsection{Single-position and consecutive gates}
\label{sec:consecutive_main}
Table~\ref{tab:consecutive_main} compares $r=1$ and $r=2$ with identical directions, threshold, and strength in each setting.
The consecutive gate fires less often and costs less utility, and the share of suppression it keeps depends on the setting: most of it on Pythia-2.8B, part of it on the fine-tuned Qwen2.5-7B tasks, and none with eight PII copies.
At matched trigger rates on Pythia-2.8B, $r=2$ reduces verbatim reproduction by 49.0 points against 45.4 for a single-position gate with a raised threshold, at equal correct-token alteration (3.34\% and 3.33\%).
We therefore recommend $r=2$ where incidental suppression is the main concern, and Table~\ref{tab:consecutive_main} shows the share of suppression it gives up in each setting (Appendix~\ref{app:consecutive_gate}).

\begin{table}[t]\centering\small
\setlength{\tabcolsep}{4pt}
\caption{Single-position ($r=1$) and two-consecutive-position ($r=2$) gates with the same directions, $\epsilon$, and $\beta$. Reproduction is the first 20 generated tokens (Pythia), the first 30 answer tokens (GSM8K), the first 50 continuation tokens (domains), or the full record (PII). All rows use $\eta=0.1$.}
\label{tab:consecutive_main}
\begin{tabular}{llrrrrrrr}
\toprule
 & & & \multicolumn{3}{c}{Reproduction $\downarrow$} & \multicolumn{3}{c}{Utility} \\
\cmidrule(lr){4-6}\cmidrule(lr){7-9}
Model & Setting & $\beta$ & None & $r=1$ & $r=2$ & None & $r=1$ & $r=2$ \\
\midrule
Pythia-2.8B & Pile & 10 & 0.665 & 0.142 & 0.175 & 22.35 & 24.90 & 24.03 \\
Qwen2.5-7B & GSM8K & 3 & 0.997 & 0.538 & 0.848 & 0.982 & 0.777 & 0.940 \\
Qwen2.5-7B & Healthcare & 10 & 0.913 & 0.027 & 0.600 & 1.089 & 20.885 & 1.768 \\
Qwen2.5-7B & Finance & 10 & 0.933 & 0.054 & 0.564 & 1.098 & 22.502 & 1.776 \\
Qwen2.5-7B & Academia & 10 & 0.919 & 0.094 & 0.477 & 1.464 & 3.006 & 1.814 \\
Qwen2.5-7B & PII, 3 copies & 10 & 0.77 & 0.01 & 0.35 & 0.610 & 0.220 & 0.467 \\
Qwen2.5-7B & PII, 8 copies & 10 & 1.00 & 0.84 & 1.00 & 0.600 & 0.473 & 0.570 \\
\bottomrule\end{tabular}\end{table}

\begin{figure}[t]
    \centering
    \begin{minipage}[b]{0.48\columnwidth}
        \centering
        \includegraphics[width=\linewidth]{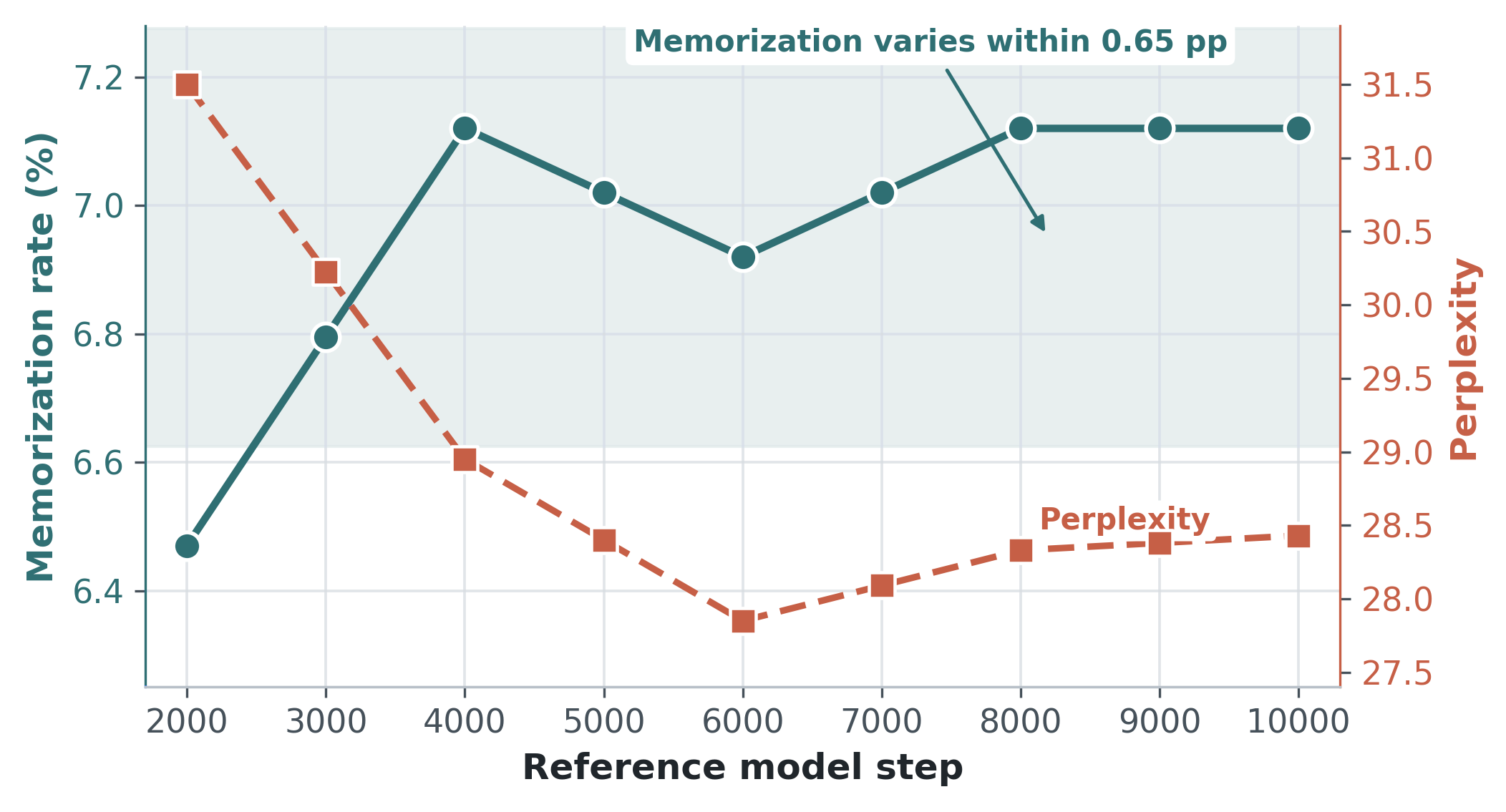}
        \vspace{-0.6em}
        \caption*{\small (a) Sensitivity of Reference Model on Pythia-2.8B.}
    \end{minipage}
    \hfill
    \begin{minipage}[b]{0.48\columnwidth}
        \centering
        \includegraphics[width=\linewidth]{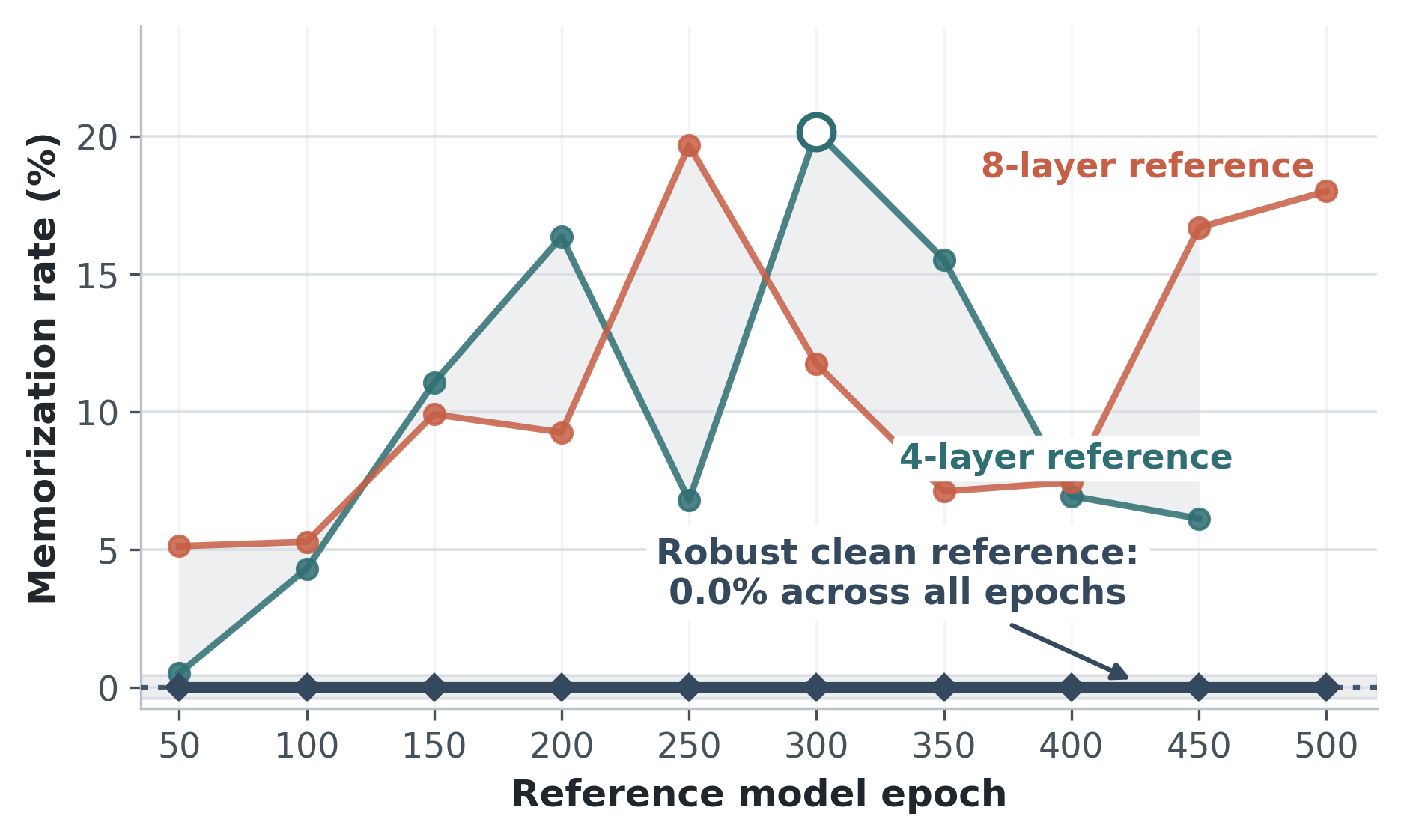}
        \caption*{\small (b) Sensitivity of Reference Model on TinyMem.}
    \end{minipage}
    \caption{
    \textbf{Reference model ablation.}
    GSS is robust to the choice of reference model.
    Across reference checkpoints, the memorization rate remains around $6.5$--$7.1\%$.
    }
    \label{ref_sensitivity}

    \vspace{0.6em}
    \includegraphics[width=0.30\linewidth,trim={0 1em 0 0},clip]{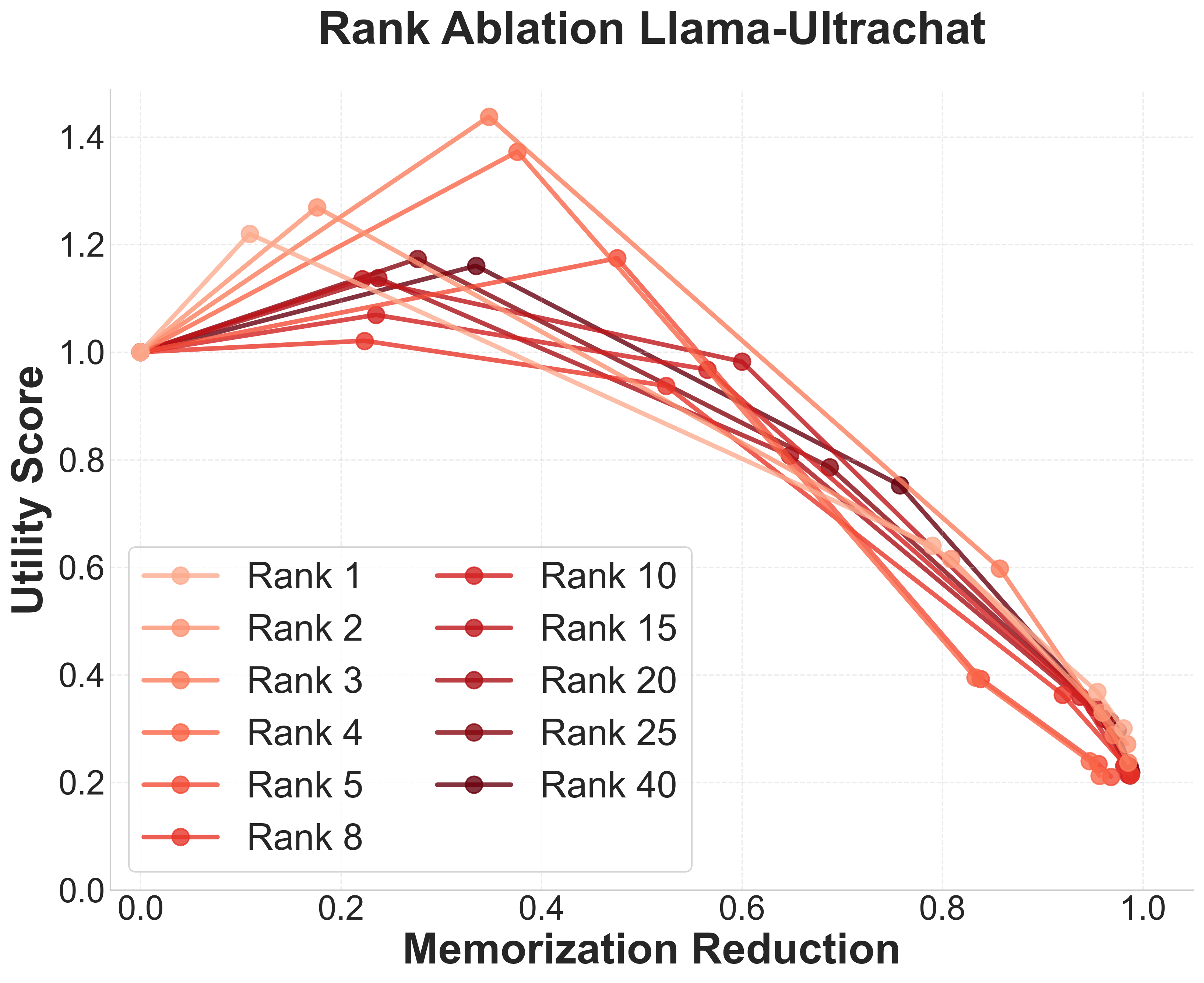}
    \hspace{2em}
    \includegraphics[width=0.30\linewidth,trim={0 1em 0 0},clip]{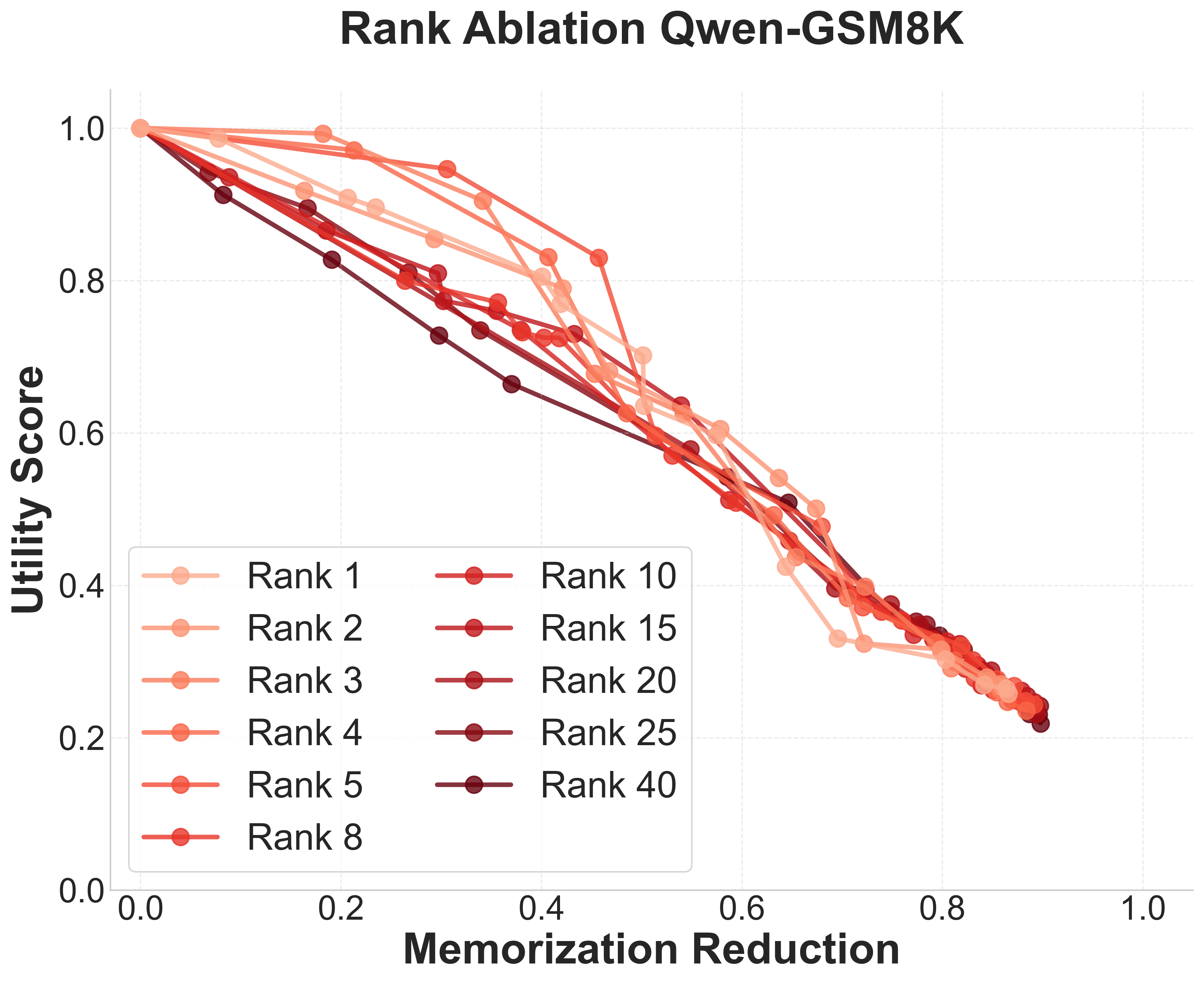}
    \vspace{-0.4em}
    \caption{
    \textbf{Visualize the Pareto frontier}. Utility score vs. memorization reduction ($r=1$).
    }
    \label{fig:ablation_rank}
    \vspace{-0.8em}
\end{figure}

\subsection{Ablation study on rank choices}

We analyze the impact of subspace rank $k$ on the mitigation and utility trade-off in Fig.~\ref{fig:ablation_rank}. The rank determines how many memorization-related directions are included in the steering subspace. 
All ranks perform similarly in the low-utility region, and \textbf{lower ranks ($k < 5$)} keep higher utility in the high-utility regime. 
This pattern is consistent with memorization-related directions occupying a low-dimensional subspace, and it favors compact steering over high-rank intervention. 

\subsection{Reference model ablation}
\label{app:ref_model_ablation}

As shown in Figure~\ref{ref_sensitivity}, changing the reference checkpoint from 2k to 10k steps only slightly changes the final memorization rate, which remains around $6.5$--$7.1\%$.
This suggests that GSS does not require a tuned reference checkpoint.
We also compare noisy and clean references. Clean references consistently give near-zero memorization.
Overall, GSS is robust to the choice of reference model.

\section{Conclusion}
We study selective activation interventions for reducing
LLM memorization.
A target--reference likelihood signal guides offline
calibration, while learned activation probes control
corrections during generation.
GSS combines closed-form spectral calibration with
configurable gates and leaves model weights unchanged.
Our experiments demonstrate substantial memorization
suppression with setting-dependent utility costs.
Signal validity depends on reference suitability,
and gate configuration changes the suppression--utility
trade-off.

\begin{ack}
This work is supported, in part, by the Natural Sciences and Engineering Research Council of Canada (NSERC) through the Discovery Grant (RGPIN-2022-05316) and Alliance Grant (ALLRP 602633-24); Mitacs (IT34007); the IITP grants (RS-2024-00445087, RS-2025-25464461); the Gemini Research Program; the Amazon Research Award; the NVIDIA Academic Grant Program; the Terry Fox Research Institute and Health Research BC; the Canada CIFAR AI Chairs program; and the Canada Research Chair program.
Furthermore, we thank UBC Advanced Research Computing and the Digital Research Alliance of Canada for providing invaluable computational resources.
The authors declare no competing interests.
\end{ack}

\bibliography{neurips_2026}
\bibliographystyle{plain}

\newpage
% Revised appendix: replace the existing appendix with this file.
% Preserve the main paper's packages, figure assets, and bibliography.
% Four new BibTeX entries are provided as comments at the end of this file.
% Search MUST_VERIFY before submission. These comments are not printed.
% Existing numerical table entries are preserved. Added reference diagnostics
% reproduce author-reported rebuttal results, not a new experimental verification.
% MUST_VERIFY [GLOBAL]: Confirm each added diagnostic against its source log;
% fill experimental metadata at the local markers. This is a review draft.
\clearpage
\appendix
\onecolumn
\section*{Appendix contents}
\startcontents[appendices]
\printcontents[appendices]{}{0}{\setcounter{tocdepth}{2}}

\FloatBarrier
\section{Theoretical foundations and scope}
\label{app:theory}
This appendix collects the derivations behind GSS, the experimental protocols,
and additional results. Throughout, $u$ denotes a probe direction and $v$ a
steer direction, following the main text.

\subsection{Excess code length and token-wise decomposition}
\label{app:proof_excess_codelen}
For fixed target and reference distributions, the model-based code lengths of
$x_{1:T}$ are $-\log p_\theta(x_{1:T})$ and $-\log p_{\mathrm{ref}}(x_{1:T})$.
This interpretation follows the prediction--compression connection studied by
Del\'etang et al.~\cite{deletang2024language}. Their difference is
\begin{equation}
S(x)=\log p_\theta(x)-\log p_{\mathrm{ref}}(x)
=\sum_{t=1}^{T}\left[\log p_\theta(x_t\mid x_{<t})-
\log p_{\mathrm{ref}}(x_t\mid x_{<t})\right]=\sum_t\omega_t.
\end{equation}
This decomposition is exact for autoregressive models evaluated on the same
sequence and compatible tokenization. Models with different tokenizers require
an explicitly aligned scoring protocol; token-wise differences are not directly
comparable otherwise. Logarithms are natural unless stated otherwise.

For any data distribution $q$ for which the quantities exist,
\begin{equation}
\label{eq:cross_entropy_identity}
\mathbb E_q[-\log p(x)]=H(q)+\mathrm{KL}(q\Vert p),\qquad
\mathbb E_q[S(x)]=\mathrm{KL}(q\Vert p_{\mathrm{ref}})
-\mathrm{KL}(q\Vert p_\theta).
\end{equation}
Positive excess likelihood can therefore reflect a better fit to reusable
structure, sequence-specific exposure, or both; it is not by itself a measure of
unintended memorization. Its usefulness depends on a meaningful reference, which
we assess empirically in Appendix~\ref{app:ref_sensitivity}.

The information-theoretic motivation does not turn this likelihood difference
into an exact mutual-information estimator. In particular, the conditional
mutual-information identity is
\[
I(X;\Theta\mid\Theta_{\mathrm{ref}})
=H(X\mid\Theta_{\mathrm{ref}})-H(X\mid\Theta,\Theta_{\mathrm{ref}}),
\]
which generally cannot drop $\Theta_{\mathrm{ref}}$ from the second term.
Likewise, model code length is not an exact evaluation of Kolmogorov complexity.
We use excess likelihood as an operational signal rather than asserting either
of these equalities for a particular pair of trained models.

\paragraph{Positive mass is a different quantity.}
Let $m_t=[\omega_t]_+$ and $\mathcal M_+(x)=\sum_t m_t$.
Then $[S(x)]_+\le\mathcal M_+(x)$, with inequality possible because
negative contributions cancel positive ones in $S(x)$ but not in $\mathcal M_+$.
The concentration analysis concerns $\mathcal M_+$; sequence-level exact
reproduction remains the mitigation outcome.

\subsection{Validity and limitations of the signal}
\label{app:proxy_validity}
A domain-competent reference with lower exposure helps separate
exposure-related predictive gains from predictability shared by both models;
it need not be more accurate than the target. High-signal positions are
therefore candidates rather than certified memorized tokens, and low-signal
positions are not certified generalized tokens. Task learning, forgetting,
model capacity, optimization noise, and reference exposure can all affect the
score. Appendix~\ref{app:token_signal_validation} reports exposure controls.

Whitening normalizes a second moment but does not remove false positives.
Singular-vector perturbation bounds would require explicit perturbation-norm and
spectral-gap conditions, so we do not rely on them for robustness to the
choice of reference.

\paragraph{Relation to membership and counterfactual analyses.}
Likelihood-ratio tests also arise in membership inference, where inclusion and
exclusion models calibrate example-specific prediction difficulty~\cite{carlini2022membership}.
GSS does not implement that membership test: its target--reference token score
is not a calibrated membership probability. Likewise, inclusion-controlled
validation is motivated by counterfactual approaches to memorization~\cite{feldman2020neural},
but exposure sensitivity alone does not establish whether a particular
prediction is harmful or whether it will lead to verbatim generation.

\subsection{Calibration gradients and accumulation}
\label{app:calibration_gradients}
For a batch $\mathcal B$, let $w_t^a$ be the fixed high- or low-signal mask,
including valid prediction positions, and $n_a=\sum_{\mathcal B}w_t^a$.
For $n_a>0$, we differentiate
\[
\mathcal L_a=-n_a^{-1}\sum_{\mathcal B}w_t^a
\log p_\theta(x_t\mid x_{<t}),\qquad g_t^a=\nabla_{h_t}\mathcal L_a.
\]
The target weights are not updated. The reference and masks are not
backpropagated through. The main-text gradient $g_t=\nabla_{h_t}\ell_{\mathcal H}$
is the high-signal log-likelihood gradient, so $g_t=-g_t^{\mathrm{hi}}$ up to
this normalization, and $M_{\mathrm{hi}}$ in Section~\ref{sec:method}
corresponds to $-\sum_t h_t (g_t^{\mathrm{hi}})^\top$. At fixed positive mask, the CE gradient is the negative
of the positive likelihood-gap gradient, up to normalization. Consequently,
$\Delta\mathcal L_a\simeq (g_t^a)^\top\Delta h_t$ describes an increase
in CE only when this inner product is positive. A positive singular-value
objective alone does not establish that property for a subtractive update.

Gradients are normalized within each batch or document, so the accumulated
matrices are sums of these normalized terms rather than uniform expectations
over positions. Gradients at a position also include downstream causal
dependencies and are therefore not, in general, gradients of that position's
next-token loss alone.

\paragraph{Contrastive constructions.}
The contrastive configurations accumulate
$M^{a}=\sum_t c^a_t\, h_t (g^a_t)^\top$ for $a\in\{\mathrm{hi},\mathrm{lo}\}$
over the calibration sequences and set $M_{\mathrm{diff}}=M^{\mathrm{hi}}-M^{\mathrm{lo}}$,
with the CE gradients $g^a_t$ defined above and no whitening. The two configurations
differ in the accumulation weight $c^a_t$. C sets $c^a_t=1$ at every valid
position, so the high- or low-signal mask enters only through the loss
$\mathcal L_a$. C-M sets $c^a_t=w^a_t$ and accumulates only at the masked
positions. Its masks cover the answer tokens of the fine-tuning items for
Qwen2.5-7B on GSM8K and synthetic PII, and every token of the memorized and
control documents for the real-document domains. The probe is the leading right
singular vector of $M_{\mathrm{diff}}$ and the steer the leading left singular
vector (Appendix~\ref{app:derivation}). In the experiments of
Appendices~\ref{app:consecutive_gate} and~\ref{app:modern_7b}, both vectors are
negated together when needed, so that the mean activation on high-signal
positions minus the mean activation on low-signal positions has a positive probe
response.

\subsection{Proof of the whitened rank-$K$ spectral optimum}
\label{app:derivation}
Let $C=\Sigma_{\mathrm{lo}}+\gamma I=LL^\top\succ0$ and
$A=L^{-1}M_{\mathrm{hi}}$. The spectral problem of the main text is
\begin{equation}
\label{eq:proof_primal}
\max_{U,V}\operatorname{tr}(U^\top M_{\mathrm{hi}}V)
\quad\text{s.t.}\quad U^\top C U=I_K,\quad V^\top V=I_K.
\end{equation}
Set $Q=L^\top U$. This maps the constraints to $Q^\top Q=I_K$ and
$V^\top V=I_K$, and the objective to $\operatorname{tr}(Q^\top A V)$.
For any feasible $V$, maximizing over $Q$ gives the nuclear norm $\|AV\|_*$.
The $K$ singular values of $AV$ are bounded by the corresponding leading
singular values of $A$, so
\[
\operatorname{tr}(Q^\top A V)\le\|AV\|_*
\le\sum_{k=1}^{K}\sigma_k(A).
\]
If $A=\widetilde U S\widetilde V^\top$, choosing
$Q=\widetilde U_{:,1:K}$ and $V=\widetilde V_{:,1:K}$ attains equality.
Hence $U=L^{-\top}\widetilde U_{:,1:K}$ is optimal, proving
Theorem~\ref{thm:optimal_solution}. Degenerate singular values can yield
non-unique optimal bases. The result concerns the spectral objective only.

For contrastive $M_{\mathrm{diff}}=P S Q^\top$, the same variational argument
maximizes $\operatorname{tr}(V^\top M_{\mathrm{diff}}U)$ at $V=P_{:,1:K}$,
$U=Q_{:,1:K}$ under Euclidean orthogonality. Here the left and right singular
vectors serve as steer and probe, respectively; this construction is not
equivalent to the whitened one.

\subsection{A fixed-direction gate perturbation bound}
\label{app:gated_proof}
\label{app:gated_interpretation}
For fixed directions $u,v$, a random activation $h$, and a gradient $g$, define
$J_0=\mathbb E[(u^\top h)(g^\top v)]$ and
$J_\epsilon=\mathbb E[\mathbb I(|u^\top h|>\epsilon)(u^\top h)(g^\top v)]$.
If $\mathbb E|g^\top v|<\infty$, then
\begin{equation}
|J_0-J_\epsilon|
\le\epsilon\mathbb E[\mathbb I(|u^\top h|\le\epsilon)|g^\top v|]
\le\epsilon\mathbb E|g^\top v|.
\end{equation}
Indeed, the difference integrates only over $|u^\top h|\le\epsilon$,
where the activation factor is bounded by $\epsilon$.
The bound concerns fixed directions only. It does not compare against an
optimum over gated directions, give a guarantee for nonlinear generation, or
cover the one-sided excess and consecutive-position gates.

\subsection{Second-moment normalization and thresholds}
\label{app:whitening_analysis}
The whitened configuration estimates the uncentered second moment
$\Sigma_{\mathrm{lo}}=N_{\mathrm{lo}}^{-1}\sum h_t h_t^\top$;
no calibration mean is subtracted during inference.
For $\epsilon>0$, Markov's inequality applied to $(u^\top h)^2$ gives
\[
\Pr(|u^\top h|>\epsilon)\le
\frac{u^\top\Sigma_{\mathrm{lo}}u}{\epsilon^2}.
\]
This bound holds for the calibration distribution and does not by itself
guarantee utility preservation. With jitter, $L^{-1}\Sigma_{\mathrm{lo}}L^{-\top}$
need not equal $I$ exactly, and whitening does not guarantee linear separation of
high- and low-signal positions. A 90th-percentile threshold places about 10\% of
calibration responses above the threshold. With multiple probes, the per-probe
exceedance rate differs from the overall rate of intervention.

\subsection{Correction coefficients}
\label{app:alpha_derivation}
For whitened directions, the update is
$h'=h-0.1\lambda\sum_k\mathbb I(|u_k^\top h|>\epsilon)
(u_k^\top h)v_k$, so the strength in Eq.~\eqref{eq:gss} is
$\beta=0.1\lambda$. For a rank-one contrastive direction,
\[
h'=h-\beta\frac{[u^\top h-\epsilon]_+}{\max(u^\top v,10^{-3})}v.
\]
Here $[z]_+=\max(z,0)$ is the rectified linear function~\cite{fukushima2007visual}.
This one-sided, threshold-subtracted update is not an instance of
Eq.~\eqref{eq:gss}, which applies to the whitened configuration; for C and
C-M, $\beta$ denotes the scale of this update.
The base contrastive configuration uses a fixed scale of $0.2$; the other
configurations set the scale explicitly.
If $u^\top v\ge10^{-3}$, $\beta=1$, and one direction is active, the corrected
probe response equals $\epsilon$. With a clipped or negative inner product,
this conclusion is not valid. In particular, clipping a negative inner product
to a positive floor does not preserve its sign. For multiple directions,
$u_i^\top h'=u_i^\top h-\sum_k a_k u_i^\top v_k$ includes cross terms, so
the update does not in general project jointly onto all threshold boundaries.

\FloatBarrier
\section{Experimental protocols and metrics}
\label{app:experimental_setup}
\paragraph{Notation.}
We write W for the whitened configuration, C for the contrastive configuration, and C-M for its variant with position-masked accumulation, used in all Qwen2.5-7B experiments (Appendix~\ref{app:calibration_gradients}). The threshold $\epsilon$ is set to the $(1-\eta)$
quantile of absolute probe responses for W and signed probe responses for C/C-M
on calibration text, so $\eta$ is the nominal per-probe exceedance rate; $r$ is the consecutive-position requirement of the gate, and
$r=1$ denotes the single-position gate.

\paragraph{Gate configurations.}
Tables~\ref{tab:4layer_combined} and~\ref{tab:pythia_unlearning} and
Figures~\ref{fig:pareto_frontier}--\ref{fig:ablation_rank} in Section~\ref{sec:experiments}
use the single-position gate ($r=1$). Table~\ref{tab:4layer_combined} uses the
whitened configuration (W) and Table~\ref{tab:pythia_unlearning} the base
contrastive configuration (C), both with rank one. In the appendix, every result
uses $r=1$ unless its table lists $r$. Table~\ref{tab:gate_configs} lists the
tables that report the two-consecutive-position gate ($r=2$), and
Table~\ref{tab:consecutive_main} in Section~\ref{sec:consecutive_main} collects one
operating point from each setting. In each of them,
$r=2$ is compared with $r=1$ under the same model, directions, threshold, and
strength; the $r=2$ rows at $\beta\in\{20,30\}$ in
Table~\ref{tab:consecutive_7b_strength} are additional strengths without an
$r=1$ counterpart.

\begin{table}[htbp]\centering\small
\setlength{\tabcolsep}{4pt}
\caption{Gate configurations of the results that report $r=2$. All directions have rank one, and $\epsilon$ is the $(1-\eta)$ quantile of probe responses on calibration text.}
\label{tab:gate_configs}
\begin{tabular}{llllllll}
\toprule
Table & Model & Data & Directions & Layer & $\eta$ & $\beta$ & $r$ \\
\midrule
\ref{tab:consecutive_pythia} & Pythia-2.8B & Pile & C & 22 & 0.1 & 3, 10 & 1, 2 \\
\ref{tab:consecutive_pythia_cohort} & Pythia-2.8B & Pile & C & 22 & 0.1 & 10 & 1, 2, 3 \\
\ref{tab:consecutive_7b} & Qwen2.5-7B & GSM8K & C-M & 21 & 0.1 & 3 & 1, 2 \\
\ref{tab:consecutive_7b_strength} & Qwen2.5-7B & GSM8K & C-M & 21 & 0.1 & 10, 20, 30 & 1, 2 \\
\ref{tab:consecutive_7b_cohort} & Qwen2.5-7B & GSM8K & C-M & 21 & 0.1 & 3 & 1, 2, 3 \\
\ref{tab:domain_grid} & Qwen2.5-7B & Real documents & C-M & 21 & 0.05, 0.1 & 1, 3, 10 & 1, 2 \\
\ref{tab:pii_duplication} & Qwen2.5-7B & Synthetic PII & C-M & 21 & 0.1 & 3, 10 & 1, 2 \\
\bottomrule\end{tabular}\end{table}

\subsection{Evaluation data and metrics}
\label{app:pythia_extraction}
On Pythia, we use $n=72$, $k=32$, and a 40-token suffix. The main experiments
use a pool of 1,632 sequences from The Pile~\cite{gao2020pile}; several
additional experiments use a separate pool of 505 extracted sequences together
with random Pile controls. Because the two populations differ, we do not
compare baseline rates across them.

On Qwen2.5-7B, we measure exact reproduction of the first 30 response tokens on
exposed training items and final-answer accuracy on 300 test items, generating
up to 160 tokens. The real-document experiments measure exact reproduction of
the first 50 continuation tokens after a 50-token prefix. We report these outcomes under
their own names rather than as a single $(n,k)$ metric, and we distinguish
teacher-forced token matches from matches in free-running generation.

\subsection{Baseline mechanisms and calibration access}
\label{app:baseline_methods}
\label{app:hyperparameters}
During calibration, GSS requires a reference model, target gradients, and
calibration examples, but it updates no parameters.
Table~\ref{tab:unlearning_assumptions} summarizes the requirements of each
baseline. The LUNAR implementation we evaluate optimizes a down-projection
matrix, so we classify it as a weight-updating method.

\begin{table}[htbp]\centering\small
\caption{Weight updates and data access required by each method as evaluated.}
\label{tab:unlearning_assumptions}
\begin{tabular}{llp{6.3cm}}
\toprule
Method & Target weight update & Calibration / optimization access \\
\midrule
GSS & No & Target, reference, calibration corpus \\
NPO implementation & Yes & Forget loss, frozen reference scores, retain loss \\
Task Vector & Yes & Fine-tuning corpus and parameter copies \\
LUNAR implementation & Yes & Unlearning/reference activations \\
ROME / MEMIT & Yes & Edit requests and implementation-specific statistics \\
\bottomrule\end{tabular}\end{table}

For NPO, we use AdamW with learning rate $10^{-5}$,
batch size 16, preference parameter 0.1, retain weight 50, and up to five epochs,
with early stopping when retain loss exceeds three times its initial value.
For Task Vector, we use AdamW with learning rate $10^{-2}$, batch size 32,
and ten epochs.

% MUST_VERIFY [BASELINES]: Restore the EMMET/PMET/AlphaEdit comparison only
% after confirming the model, dataset, rate units, split, hyperparameters,
% and comparable operating points. Rebuttal rows (Mem after / clean PPL):
% Baseline 6.50 / 321; MEMIT 0.00 / 345; EMMET 2.20 / 342;
% PMET 6.70 / 338; AlphaEdit 0.00 / 381; GSS 0.00 / 336.
% These rows are deliberately not active LaTeX: their setting is missing.
% MUST_VERIFY [TIMING]: Recover hardware, batch size and timing boundaries
% before reporting total calibration+application cost. Rebuttal totals:
% TinyMem 0.45--0.51 s; Pythia-2.8B 10.54 s; Pythia-6.9B 15.39 s.
% Do not compare total baseline optimization with only GSS's application time.

\subsection{Uncertainty and observed zeros}
\label{app:uncertainty}
We use sequence-level bootstrap resampling because positions within a
sequence are dependent. Seed ranges and standard deviations describe variation
across seeds and are not confidence intervals. For $z=0$ observed
successes among $N$ independent Bernoulli trials, the two-sided 95\% exact
binomial upper limit is $1-0.025^{1/N}$; the one-sided 95\% upper limit is
$1-0.05^{1/N}$. For $N=998$, these are approximately 0.369\% and 0.300\%,
respectively. 
A bootstrap over an all-zero vector yields a degenerate interval and does not
establish zero population risk. Because repeated attacks on the same sequence
are correlated, we summarize attack success per sequence at the stated budget.

\FloatBarrier
\section{Signal validation and reference dependence}
\label{app:additional_results}
\subsection{Concentration of positive likelihood mass}
\label{app:additional_pilot_results}
For nonnegative masses $m_i=[\omega_i]_+$ with total $S_+>0$, define the share
carried by the largest 10\% of positions, the Gini coefficient
$\sum_{i,j}|m_i-m_j|/(2N S_+)$, and normalized effective support
$S_+^2/(N\sum_i m_i^2)$. Zero-mass positions remain in the denominator.
For $S_+=0$, concentration is undefined rather than evidence of sparsity.
Positive-position frequency is reported separately from mass concentration.

\begin{table}[htbp]\centering\small
\caption{Concentration of positive likelihood mass: ranges across twelve settings (11,833 sequences, 1.72 million positions).}
\label{tab:concentration_reported}
\begin{tabular}{lr}
\toprule
Statistic & Range \\
\midrule
Top-10\% positive-mass share & 30--89\% \\
Relative to uniform allocation & $3.0$--$8.9\times$ \\
Gini coefficient & 0.50--0.91 \\
Normalized effective support & 0.088--0.532 \\
Positions with $\omega>0$ & 47--94\% \\
\bottomrule\end{tabular}\end{table}

These statistics indicate nonuniform positive mass rather than a specific
parametric heavy tail: exponential tails fit better than power-law tails in
nine of twelve settings. Theorem~\ref{thm:optimal_solution} assumes neither
the tail shape nor concentration.

On Qwen3--GSM8K, the signal has skewness 2.89, a mean positive-run length of
1.64, and 60.8\% singleton runs. These describe runs of positive signal, not
lengths of memorized spans.

\begin{figure}[htbp]\centering
\begin{subfigure}[t]{0.49\linewidth}\centering
\includegraphics[width=\linewidth]{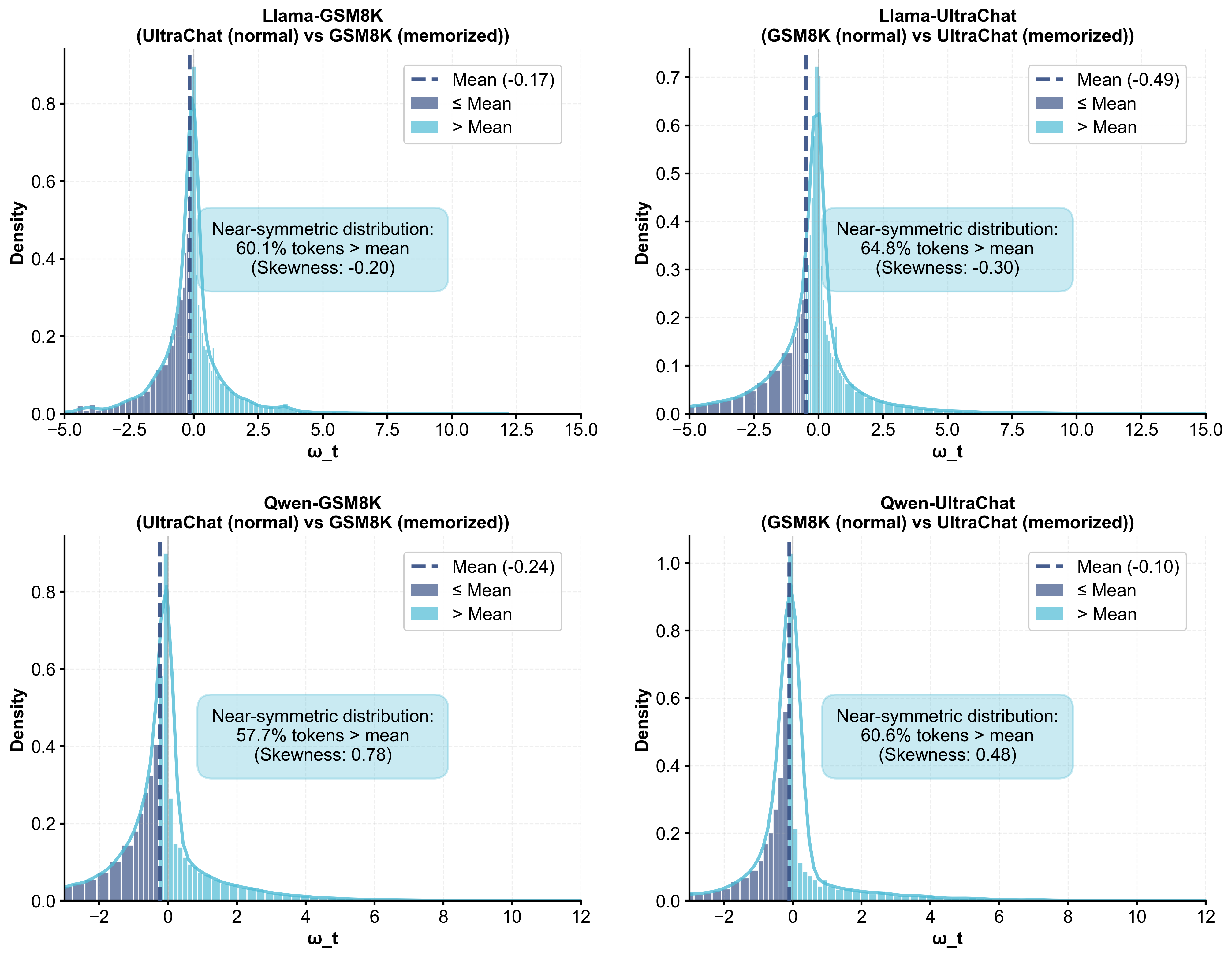}
\caption{Signal distributions on control data.}
\label{fig:app-normal}
\end{subfigure}\hfill
\begin{subfigure}[t]{0.49\linewidth}\centering
\includegraphics[width=\linewidth]{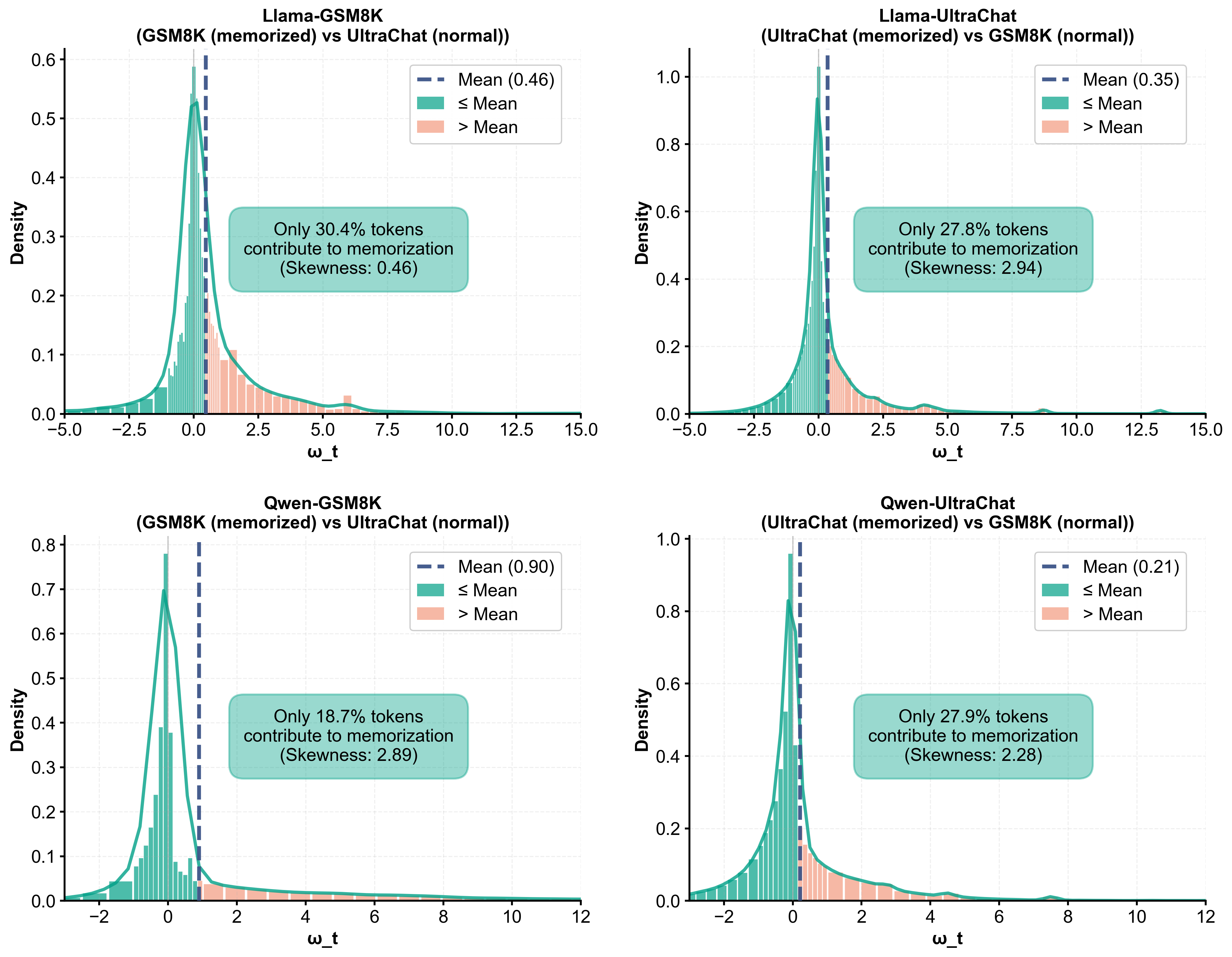}
\caption{Signal distributions on reproduced sequences.}
\label{fig:app-mem}
\end{subfigure}

\vspace{0.6em}
\begin{subfigure}[t]{0.6\linewidth}\centering
\includegraphics[width=\linewidth]{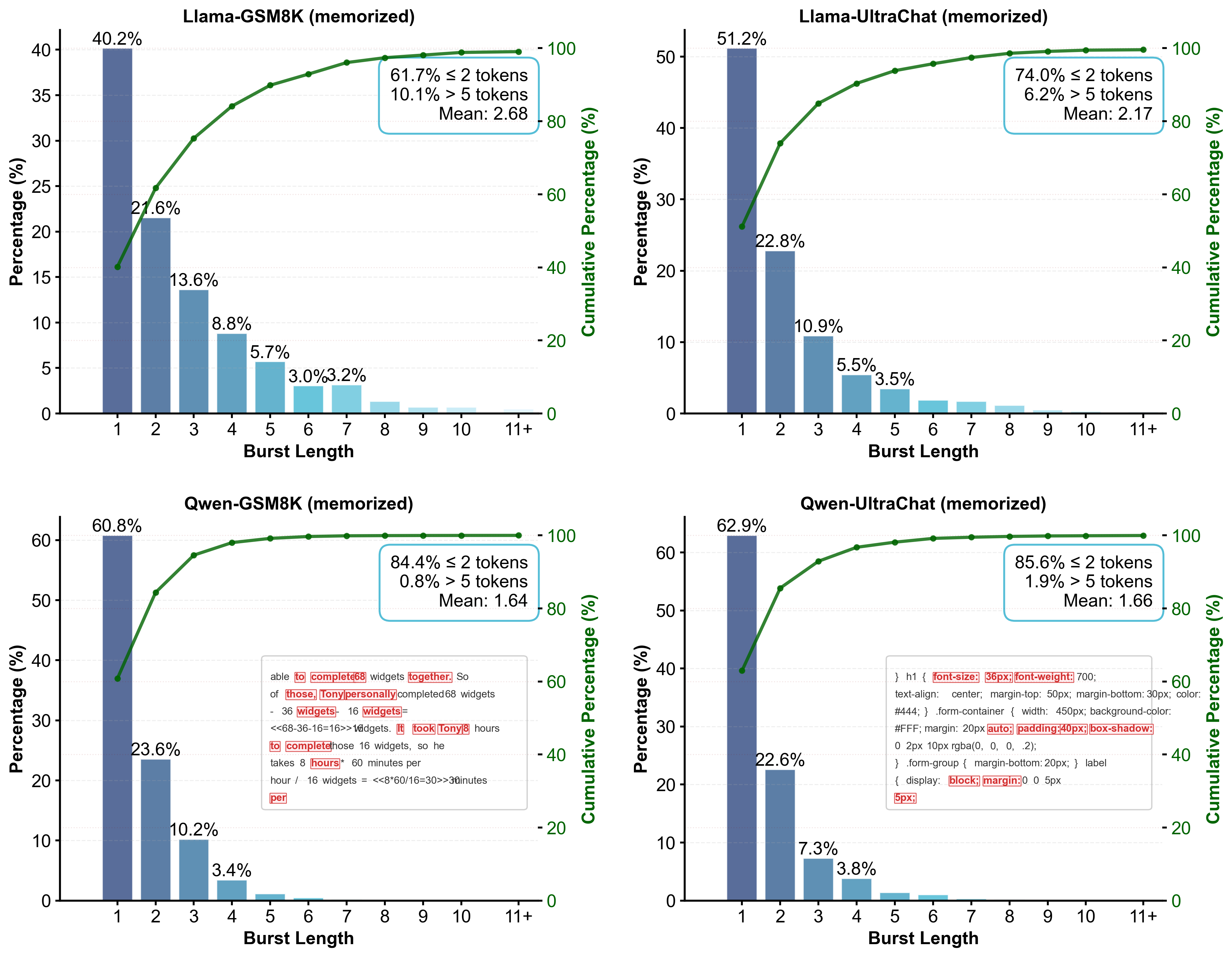}
\caption{Run lengths of positive signal.}
\label{fig:app-burst}
\end{subfigure}
\caption{Distribution of the likelihood signal $\omega_t$. Low signal is not a generalization label, distributional shape is setting-dependent, and run frequency does not measure intervention magnitude or causal importance.}
\label{fig:image2}
\end{figure}

% MUST_VERIFY [CONCENTRATION]: Add the per-setting model/dataset/count/statistic
% rows from the 12-setting analysis; only aggregate ranges are supplied here.

\subsection{Exposure validation versus verbatim ranking}
\label{app:token_signal_validation}
% MUST_VERIFY [CANARIES]: Add inclusion/exclusion training details, numbers of
% sequences and seeds, token-label construction, and score aggregation.
In seed-controlled canary
experiments, $\omega$ identifies the top decile of inclusion gain with AUROC
0.99, and Pythia-2.8B scores reach 0.95 against independent Pythia-6.9B labels;
the corresponding argmax-flip AUROCs are 0.83 and 0.81. On a separate pool of
505 extracted and 1,024 control sequences, positive-$\omega$ positions are about
7.67 times as likely as nonpositive ones to begin a 20-token verbatim
continuation.

\paragraph{Partition ablation.}
On TinyMem, we replace the high-signal calibration subset with a random or an
inverse subset. This tests how useful the selected subset is for constructing
an intervention, not whether it identifies memorized tokens.

\begin{table}[t]
\centering
\caption{TinyMem partition ablation. Calibrating on high-signal positions gives the lowest reproduction.}
\label{tab:token_signal_validation}
\scriptsize
\renewcommand{\arraystretch}{1.12}
\setlength{\tabcolsep}{3pt}
\resizebox{\linewidth}{!}{%
\begin{tabular}{l ccc ccc ccc ccc}
\toprule
\multirow{2}{*}{Partition}
& \multicolumn{3}{c}{Math Noise}
& \multicolumn{3}{c}{Math Backdoor}
& \multicolumn{3}{c}{LM Noise}
& \multicolumn{3}{c}{LM Backdoor} \\
\cmidrule(lr){2-4} \cmidrule(lr){5-7}
\cmidrule(lr){8-10} \cmidrule(lr){11-13}
& Mem. $\downarrow$ & Acc. $\uparrow$ & Time $\downarrow$
 & Mem. $\downarrow$ & Acc. $\uparrow$ & Time $\downarrow$
 & Mem. $\downarrow$ & PPL $\downarrow$ & Time $\downarrow$
 & Mem. $\downarrow$ & PPL $\downarrow$ & Time $\downarrow$ \\
\midrule
Random partition
& 18.42 & 96.71 & 0.001
& 21.36 & 96.44 & 0.001
& 27.85 & 64.02 & 0.002
& 31.17 & 64.36 & 0.003 \\
Inverse partition $(\omega_t < 0)$
& 34.79 & 96.55 & 0.001
& 39.28 & 96.21 & 0.001
& 46.63 & 64.71 & 0.002
& 49.85 & 65.08 & 0.003 \\
\textbf{High-$\omega_t$ partition (GSS)}
& \textbf{0.00} & \textbf{96.98} & 0.001
& \textbf{0.00} & \textbf{96.82} & 0.001
& \textbf{0.00} & \textbf{63.13} & 0.002
& \textbf{0.00} & \textbf{63.17} & 0.003 \\
\bottomrule
\end{tabular}}
\end{table}

\subsection{Reference exposure controls}
\label{app:ref_sensitivity}
We separate two questions: whether the reference provides an exposure-sensitive
signal, and whether the calibrated intervention remains effective across a
particular set of reference checkpoints. Similar post-intervention reproduction
rates do not, by themselves, establish that the signal measures exposure.

\paragraph{References without the original checkpoint.}
The reference need not be the target's original or immediately preceding
checkpoint. A retained pre-exposure checkpoint, an independently trained model,
or a domain-competent model trained on disjoint data can provide a useful
comparison when it has lower exposure to the evaluated sequences. Compatible
tokenization is required for the token-wise comparison used here.

% MUST_VERIFY [REF-CONTROL]: Results below are transcribed from the rebuttal.
% Add the exact target/reference identifiers, corpus, sample count, definition
% of the sequence label, aggregation of omega, seed count and clone-noise scale.
\begin{table}[htbp]\centering\small
\caption{Controlled reference comparison: sequence-level discrimination.
These finite-sample results test exposure contrast rather than universal
reference interchangeability. No intervention gate is applied in this diagnostic.}
\label{tab:reference_exposure_controls}
\begin{tabular}{p{7.2cm}cc}
\toprule
Reference & Exposure contrast & AUC \\
\midrule
Matched unexposed model & Yes & 1.00 \\
Independent-seed unexposed model & Yes & 1.00 \\
Model trained on disjoint data & Yes & 1.00 \\
Noise-perturbed target clone & No & 0.58 \\
\bottomrule
\end{tabular}
\end{table}
Table~\ref{tab:reference_exposure_controls} compares references with exposure
contrast against a perturbed copy of the target. The weaker performance of
the clone supports the role of exposure contrast in this control; disagreement
introduced by perturbation is not sufficient to recover the same discrimination.
An observed AUC of 1.00 does not guarantee perfect discrimination on new data.

\paragraph{Exposure dose response.}
% MUST_VERIFY [REF-SWEEP]: Recover the checkpoint-by-checkpoint rows and their
% exact step labels from the source log. The supplied rebuttal provides only
% these endpoints; do not infer step IDs from the intervention sweep below.
As the reference becomes increasingly exposed to the evaluated sequences,
verbatim-20 enrichment decreases from approximately $36.9\times$ to $1.4\times$,
and sequence-level AUC decreases from 0.92 to 0.49. This diagnostic is distinct
from the post-intervention sweep in Table~\ref{tab:ref_pythia}. Enrichment near
one indicates little localization advantage under that reference and evaluation
protocol. Thus, lower exposure is an empirical requirement for the intended
interpretation of the signal, rather than robustness to arbitrary references.

In an exploratory control of predictable but unexposed text, the
false-positive rate at 90\% recall is 16.5\% for target confidence and 0\% for
$\omega$. Unlike target confidence, $\omega$ thus separates exposure from mere
predictability in this experiment.

\paragraph{Post-intervention checkpoint sensitivity.}
Table~\ref{tab:ref_pythia} reports reproduction and PPL
after intervention as the reference checkpoint varies.

\begin{table}[h]
\centering
\small
\caption{Pythia-2.8B reproduction and PPL across reference checkpoints ($r=1$).}
\label{tab:ref_pythia}
\begin{tabular}{ccc}
\toprule
Reference Model Step & Mem.\ Rate & PPL \\
\midrule
2000 & 6.45\% & 31.49 \\
4000 & 7.11\% & 29.49 \\
6000 & 6.93\% & 28.26 \\
8000 & 7.11\% & 28.64 \\
10000 & 7.11\% & 28.74 \\
\bottomrule
\end{tabular}
\end{table}

\FloatBarrier
\section{Gate configuration, selectivity, and utility}
\label{app:consecutive_gate}
All consecutive-gate ($r=2$) experiments in this section use contrastive
configurations, C on Pythia-2.8B and C-M on Qwen2.5-7B. They characterize the run
requirement under these configurations and do not directly validate the whitened configuration analyzed in Section~\ref{sec:method}.

\subsection{Pythia-2.8B: consecutive gate comparison}
The target is Pythia-2.8B-deduped at step 143000 and the reference at step 2000.
We estimate rank-one C directions at layer 22 (zero-indexed) on 70\% of the
extracted pool and evaluate on the remaining 30\%, for split seeds 42, 43, and
44. The threshold is the 90th percentile of the signed projection on the last
1,024 random Pile calibration sequences, and PPL uses the first 800. The table
reports means over the three splits; the target model is not retrained.

\begin{table}[htbp]\centering\small
\caption{Pythia-2.8B, rank-one C, $\eta=0.1$. Trigger is the fraction of evaluated positions at which the gate fires; number preservation measures overlap with numbers in the reference continuation and is not task accuracy.}
\label{tab:consecutive_pythia}
\begin{tabular}{lrrrrr}
\toprule
$\beta$ & $r$ & Exact & PPL & Trigger & Number preservation \\
\midrule
3 & 1 & 0.1667 & 23.45 & 0.435 & 0.434 \\
3 & 2 & 0.1732 & 23.15 & 0.325 & 0.439 \\
10 & 1 & 0.0965 & 25.04 & 0.379 & 0.333 \\
10 & 2 & 0.1118 & 24.15 & 0.271 & 0.379 \\
\bottomrule\end{tabular}\end{table}

At $\beta=10$, requiring two consecutive exceedances lowers the trigger rate
and PPL but raises residual exact reproduction from 0.0965 to 0.1118, so the
run requirement trades suppression for utility. Residual rates reported
elsewhere in this appendix (e.g., Table~\ref{tab:mask_controls}) use a
different pool and protocol and are not directly comparable.

\subsection{Qwen2.5-7B: gate transfer and strength sweep}
On Qwen2.5-7B, we compare the two gates at matched strength on GSM8K with
rank-one C-M directions, $\eta=0.1$, and 30-token reproduction, using a single
seed. Reproduction is exact match of the first 30 response tokens on exposed
training items, and utility is final-answer accuracy on 300 GSM8K test items
(Appendix~\ref{app:pythia_extraction}).

\begin{table}[htbp]\centering\small
\caption{Qwen2.5-7B on GSM8K, C-M, $\beta=3$, $\eta=0.1$, single seed. Accuracy is measured on 300 test items.}
\label{tab:consecutive_7b}
\begin{tabular}{lrrrr}
\toprule
$r$ & Exact-30 & Accuracy & Trigger & Number preservation \\
\midrule
1 & 0.515 & 0.533 & 0.078 & 0.854 \\
2 & 0.770 & 0.600 & 0.021 & 0.937 \\
\bottomrule\end{tabular}\end{table}

The consecutive gate intervenes less often and preserves more of the output,
but substantially weakens suppression. Table~\ref{tab:consecutive_7b_strength}
reports a three-seed strength sweep. The sweep reports number preservation
rather than test accuracy, and number preservation is not a substitute for
test accuracy.

\begin{table}[htbp]\centering\small
\caption{Qwen2.5-7B gate strength sweep with C-M, $\eta=0.1$, and 30-token exact reproduction, reported per seed.}
\label{tab:consecutive_7b_strength}
\begin{tabular}{lrrrrr}
\toprule
Seed & $r$ & $\beta$ & Exact-30 & Number preservation & Trigger \\
\midrule
1 & 1 & 10 & 0.243 & 0.697 & 0.076 \\
1 & 2 & 10 & 0.622 & 0.863 & 0.021 \\
1 & 2 & 20 & 0.566 & 0.836 & 0.020 \\
1 & 2 & 30 & 0.536 & 0.814 & 0.018 \\
2 & 1 & 10 & 0.137 & 0.533 & 0.173 \\
2 & 2 & 10 & 0.387 & 0.669 & 0.100 \\
2 & 2 & 20 & 0.300 & 0.596 & 0.099 \\
2 & 2 & 30 & 0.271 & 0.571 & 0.099 \\
44 & 1 & 10 & 0.333 & 0.756 & 0.078 \\
44 & 2 & 10 & 0.654 & 0.889 & 0.021 \\
44 & 2 & 20 & 0.581 & 0.853 & 0.021 \\
44 & 2 & 30 & 0.536 & 0.833 & 0.022 \\
\bottomrule\end{tabular}\end{table}

Increasing strength does not make the two-position gate uniformly better, so
the run requirement should be chosen for each model and data setting.

\subsection{Run requirement on memorized and legitimate cohorts}
\label{app:consecutive_cohorts}
Tables~\ref{tab:consecutive_pythia_cohort} and~\ref{tab:consecutive_7b_cohort}
vary the run requirement over $r\in\{1,2,3\}$ and separate the effect of the
gate on memorized sequences from its effect on legitimate ones. Both cohorts
are fixed at baseline, before any intervention. Trigger is the share of
sequences in which the gate fires at least once during generation, and
selectivity is the memorized trigger divided by the legitimate trigger. Each
table uses a single estimate of the directions.

For Pythia-2.8B, the target, reference, and layer 22 are those of
Table~\ref{tab:consecutive_pythia}. The rank-one sign-fixed C direction is
estimated once on the full 505-sequence pool, so the memorized cohort lies
inside the estimation pool. The strength is $\beta=10$, and $\varepsilon$ is
the 90th percentile of the projection on the same Pile calibration rows as in
Table~\ref{tab:consecutive_pythia}. The memorized cohort contains the 480 pool
sequences with a teacher-forced 20-token verbatim run, and the legitimate
cohort contains 1{,}174 random Pile sequences with no such run. Decoding is
greedy for 40 tokens after a 32-token prompt, and Verbatim-20 is exact
reproduction of the first 20 generated tokens. On the legitimate cohort,
correct$\to$incorrect is the share of teacher-forced tokens (positions 16
onward) that the target predicts correctly without the gate and incorrectly
with it, and $\Delta$Acc is the paired change in teacher-forced token accuracy.

\begin{table}[htbp]\centering\small
\setlength{\tabcolsep}{4pt}
\caption{Pythia-2.8B cohorts, rank-one C, $\beta=10$, $\eta=0.1$. Values in \%, except selectivity and PPL. Verbatim-20 is 66.5 without the gate, and PPL on the first 800 random Pile sequences is 22.35.}
\label{tab:consecutive_pythia_cohort}
\begin{tabular}{lrrrrrrrr}
\toprule
& \multicolumn{3}{c}{Memorized ($n=480$)} & \multicolumn{3}{c}{Legitimate ($n=1{,}174$)} & & \\
\cmidrule(lr){2-4}\cmidrule(lr){5-7}
$r$ & Trigger & Verbatim-20 & Reduction & Trigger & Corr.$\to$incorr. & $\Delta$Acc & Selectivity & PPL \\
\midrule
1 & 96.9 & 14.2 & 52.3 & 78.4 & 5.8 & $-1.6$ & 1.23 & 24.90 \\
2 & 89.4 & 17.5 & 49.0 & 49.7 & 3.3 & $-1.0$ & 1.80 & 24.03 \\
3 & 85.2 & 22.1 & 44.4 & 31.5 & 2.0 & $-0.6$ & 2.70 & 23.53 \\
\bottomrule\end{tabular}\end{table}

For Qwen2.5-7B, the adapter, direction, layer 21, and $\beta=3$ are those of
Table~\ref{tab:consecutive_7b}. The threshold $\varepsilon$ is the 90th
percentile of the projection on GSM8K test items 300--699, which are disjoint
from the evaluated test items. The memorized cohort contains the 884 of the
1{,}000 SFT items whose first 30 answer tokens are reproduced at baseline, and
the legitimate cohort contains the 166 of the first 300 test items that are
answered correctly at baseline and share no 20-token span with any SFT answer.
Decoding is greedy for 160 tokens. Correct$\to$incorrect is the share of
baseline-correct legitimate items that become incorrect, and $\Delta$Acc is the
change in final-answer accuracy on the legitimate cohort.

\begin{table}[htbp]\centering\small
\setlength{\tabcolsep}{4pt}
\caption{Qwen2.5-7B on GSM8K cohorts, rank-one C-M, $\beta=3$, $\eta=0.1$, single seed. Values in \%, except selectivity. Without the gate, the memorized cohort has Verbatim-30 99.7 and the legitimate cohort has accuracy 98.2.}
\label{tab:consecutive_7b_cohort}
\begin{tabular}{lrrrrrrr}
\toprule
& \multicolumn{3}{c}{Memorized ($n=884$)} & \multicolumn{3}{c}{Legitimate ($n=166$)} & \\
\cmidrule(lr){2-4}\cmidrule(lr){5-7}
$r$ & Trigger & Verbatim-30 & Reduction & Trigger & Corr.$\to$incorr. & $\Delta$Acc & Selectivity \\
\midrule
1 & 100.0 & 53.8 & 45.8 & 100.0 & 21.5 & $-20.5$ & 1.00 \\
2 & 89.3 & 84.8 & 14.8 & 83.7 & 4.9 & $-4.2$ & 1.07 \\
3 & 58.0 & 95.2 & 4.4 & 48.8 & 1.8 & $-1.8$ & 1.19 \\
\bottomrule\end{tabular}\end{table}

On Pythia-2.8B, raising $r$ lowers the legitimate trigger rate faster than the
memorized one, and $r=2$ keeps 49.0 of the 52.3 points of verbatim reduction.
On Qwen2.5-7B, the two cohorts trigger at similar rates for every $r$, and $r=2$
keeps 14.8 of the 45.8 points. These tables differ from
Tables~\ref{tab:consecutive_pythia} and~\ref{tab:consecutive_7b} in four
respects: cohorts selected at baseline instead of a held-out split (Pythia) or
all training items (Qwen), sequence-level instead of token-level trigger rates,
verbatim reproduction of the first 20 tokens instead of exact match of all 40
(Pythia), and one direction estimate on the full pool instead of one per split
(Pythia) or a calibration slice disjoint from the evaluated test items (Qwen).
Under both protocols, $r=2$ moves residual reproduction little on Pythia-2.8B
(0.0965 to 0.1118 in Table~\ref{tab:consecutive_pythia}) and substantially on
Qwen2.5-7B (0.515 to 0.770 in Table~\ref{tab:consecutive_7b}).

\paragraph{Scope of the consecutive-gate recommendation.}
On Pythia-2.8B, $r=2$ retains approximately $49.0/52.3=93.7\%$ of the
suppression achieved by $r=1$, using the rounded entries in
Table~\ref{tab:consecutive_pythia_cohort}. On the Qwen2.5-7B cohort it retains
only about one third. The Pythia ratio should therefore not be generalized to
other settings. We treat $r$ as an operating-point choice: $r=2$ is suitable
only when reducing incidental intervention is a priority and its residual
reproduction rate is acceptable on validation data.

% MUST_VERIFY [RUN-DIAGNOSTIC]: Rebuttal reports strongest-position AUC 0.666,
% omega maximum-run AUC 0.952, and projection maximum-run AUC 0.882.
% Restore only with exact omega threshold, scored length, teacher-forced versus
% free-running protocol, and disjoint calibration details. Run thresholds were
% lost in the pasted rebuttal. Do not identify these with online r=2 results.
% MUST_VERIFY [COHORTS]: Explain why cohorts selected as baseline-correct /
% exactly reproduced have re-evaluated baselines 98.2% / 99.7%, not 100%.

\subsection{Token classes, mask controls, and causal tests}
\label{app:single_token_criticality}
The mask ablation uses 505 Pythia-2.8B sequences and rank-one C with
$\eta=0.1$ and $\beta=10$. Each mask restricts the gate to a subset of
positions. The oracle mask is a diagnostic only, since it requires knowing
which positions are verbatim.

\begin{table}[htbp]\centering\small
\caption{Mask ablations on Pythia-2.8B (rank-one C, $\beta=10$, $\eta=0.1$). ``No $X$'' masks exclude positions of class $X$ from the gate.}
\label{tab:mask_controls}
\begin{tabular}{lrrrr}
\toprule
Mask & Exact & PPL & Number & Entity \\
\midrule
Full gate & 0.085 & 24.69 & 0.324 & 0.323 \\
No whitespace/punctuation & 0.093 & 24.31 & 0.377 & 0.349 \\
No function words & 0.089 & 24.26 & 0.340 & 0.361 \\
No numbers & 0.089 & 24.58 & 0.340 & 0.321 \\
No entities & 0.091 & 24.48 & 0.330 & 0.357 \\
No numbers/entities/function words & 0.099 & 23.95 & 0.358 & 0.392 \\
Multi-position runs only & 0.091 & 24.02 & 0.340 & 0.345 \\
Runs of length $\geq 3$ only & 0.101 & 23.67 & 0.369 & 0.358 \\
Oracle verbatim positions & 0.097 & 22.38 & 0.378 & 0.390 \\
\bottomrule\end{tabular}\end{table}

For the counterfactual single-token analysis, we sample 200 sequences and
classify the outcomes for the 111 sequences that the unmodified model
reproduces. An outcome is labeled harmful if it loses a reference number or
entity and has low semantic similarity; this is a heuristic rather than a
blinded human judgment. The fractions in Table~\ref{tab:causal_single} are
conditional on the 111 reproduced sequences. They are distinct from the
correct-to-incorrect rates on legitimate cohorts, and should not be compared
without aligning the population and denominator.
% MUST_VERIFY [COUNTERFACTUAL]: Add the single-token selection/intervention
% protocol, class decision rules, similarity threshold and bootstrap unit.
% The separate rebuttal harmful-intervention rates 39.6/24.4/17.9% need their
% original numerator and denominator before they can be reported here.
% Also verify the 15% singleton-mass and 95%/90% burst-detection claims before
% adding them; raw omega runs and activation-probe runs are different objects.

\begin{table}[htbp]\centering\small
\caption{Counterfactual single-token classifications (rank-one C, $\beta=10$, $\eta=0.1$) with 95\% bootstrap intervals.}
\label{tab:causal_single}
\begin{tabular}{lrr}
\toprule
Class & Fraction & 95\% CI \\
\midrule
Redundant & 0.757 & [0.682, 0.836] \\
Generic/interchangeable & 0.126 & [0.068, 0.180] \\
Harmful & 0.081 & [0.038, 0.127] \\
Verbatim-critical & 0.036 & [0.000, 0.062] \\
\bottomrule\end{tabular}\end{table}

\paragraph{Matched-budget control.}
At matched trigger rates, the $r=2$ gate reduces reproduction by 49.0
percentage points versus 45.4 for a higher-threshold $r=1$ gate, with
correct-token alteration of 3.34\% versus 3.33\%.

\subsection{Threshold calibration under distribution shift}
\label{app:epsilon}
\paragraph{Meaning of the nominal exceedance rate.}
Let $z=|u^\top h|$ for W and $z=u^\top h$ for the one-sided C/C-M gate,
and let $\widehat F_{\mathrm{cal}}$ be its empirical calibration distribution.
We set
\[
\epsilon(\eta)=\widehat F_{\mathrm{cal}}^{-1}(1-\eta).
\]
The empirical exceedance rate is approximately $\eta$, up to finite-sample
quantile conventions and ties. This is not a measured false-positive rate:
calibration positions do not supply token-level memorization ground truth,
and the distribution of responses may change during generation.

\paragraph{Effect of varying $\eta$.}
For fixed directions and a fixed sequence of unmodified activations,
$\eta_2>\eta_1$ implies $\epsilon(\eta_2)\le\epsilon(\eta_1)$, so the
eligible set of positions cannot shrink. The same set-inclusion property
holds for a fixed consecutive-run requirement on that fixed response sequence.
For W, lowering the threshold admits additional correction terms; for C/C-M,
it also increases the excess coefficient $[u^\top h-\epsilon]_+$ at
already-active positions. Neither observation guarantees monotone suppression
or utility in autoregressive generation, because interventions change future
states and outputs. The run-level trigger probability need not equal $\eta^r$,
since consecutive responses are dependent.

The real-document sweep in Table~\ref{tab:domain_grid} illustrates the trade-off.
For healthcare at $\beta=1$ and $r=1$, increasing $\eta$ from 0.05 to 0.1
reduces reproduction from 0.233 to 0.053 while control PPL rises from 2.019
to 4.356. With $r=2$, reproduction decreases from 0.840 to 0.653 and PPL
rises from 1.138 to 1.344. These are setting-specific operating points, not a
monotonicity guarantee. Figure~\ref{fig:epsilon} varies the threshold directly;
its horizontal axis should not be interpreted as the quantile parameter $\eta$.

% MUST_VERIFY [ETA-SWEEP]: The rebuttal contains eta=.01/.05/.1/.2/.5.
% Its three exact-reproduction rows are:
% .404/.279/.250/.190/.145; .331/.200/.143/.105/.073;
% .250/.111/.087/.069/.053.
% Random-Pile PPL: 21.80/21.98/22.14/22.38/22.91.
% Realized clean trigger: .009/.046/.095/.192/.485.
% Beta labels and some protocol details were lost in the pasted text.
% Restore the full table only after recovering these labels from the logs.

\paragraph{Transfer of a calibrated threshold.}
Table~\ref{tab:calibration_shift} reports teacher-forced exceedance rates of
the C probe with a threshold set at the 90th percentile of signed responses on
Pile text. The threshold attains its nominal rate on held-out Pile text, but the
rate shifts with the evaluation distribution, from 0.038 on number-heavy
answers to 0.127 on UltraChat instructions.

\begin{table}[htbp]\centering\small
\caption{Probe exceedance rates under distribution shift with a Pile-calibrated threshold (nominal rate 0.1).}
\label{tab:calibration_shift}
\begin{tabular}{lrr}
\toprule
Evaluation text & Positions & Exceedance \\
\midrule
Pile (held-out) & 31600 & 0.1006 \\
UltraChat instructions & 73388 & 0.1265 \\
Long context (600 tokens) & 86580 & 0.1132 \\
Entity-heavy text & 61668 & 0.1068 \\
GSM8K & 31096 & 0.0572 \\
Number-heavy answers & 29381 & 0.0380 \\
\bottomrule\end{tabular}\end{table}

\begin{figure}[htbp]\centering
\IfFileExists{fig/epsilon_ablation_perplexity.pdf}{\includegraphics[width=0.55\linewidth]{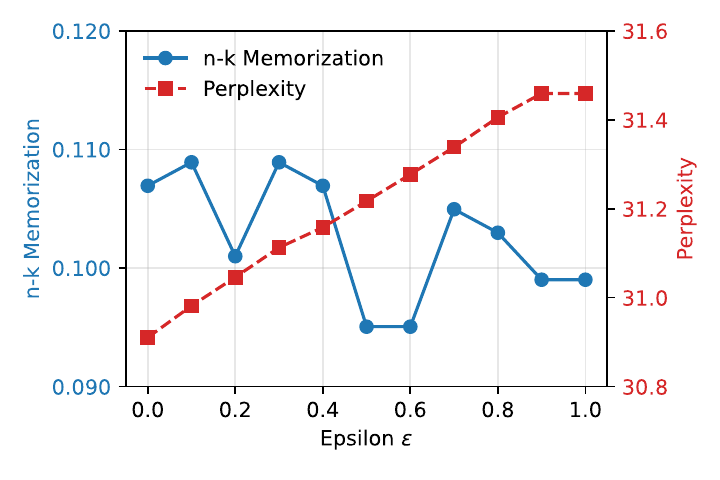}}{\fbox{\parbox{0.85\linewidth}{Original figure asset required: \texttt{\detokenize{fig/epsilon_ablation_perplexity.pdf}}.}}}
\caption{Threshold ablation on Pythia-2.8B ($r=1$).}
\label{fig:epsilon}
\end{figure}

\subsection{Gated versus ungated results}
\label{app:gating_ablation}
Table~\ref{tab:gating_ablation} lists the values behind
Figure~\ref{fig:gated_ungated_ppl}. The ungated variant uses the same
calibrated directions but removes the indicator and applies the correction at
every position. Both variants reach 0.0\% memorization on the TinyMem language
model at every noise level, and gating lowers PPL by 2.4--4.1 points.

\begin{table}[htbp]\centering\small
\caption{Gated versus ungated intervention on the TinyMem language model under increasing numbers of noisy training samples. Memorization is 0.0\% for both variants.}
\label{tab:gating_ablation}
\begin{tabular}{rrrr}
\toprule
Noisy samples & Gated PPL $\downarrow$ & Ungated PPL $\downarrow$ & $\Delta$ \\
\midrule
500  & 58.97 & 62.17 & $-3.2$ \\
800  & 60.40 & 62.90 & $-2.5$ \\
1000 & 59.37 & 63.17 & $-3.8$ \\
1300 & 62.20 & 65.10 & $-2.9$ \\
1600 & 61.60 & 65.70 & $-4.1$ \\
1900 & 63.30 & 66.30 & $-3.0$ \\
2200 & 66.40 & 68.80 & $-2.4$ \\
2500 & 66.90 & 70.50 & $-3.6$ \\
\bottomrule\end{tabular}\end{table}

\FloatBarrier
\section{Additional models, domains, and utility measures}
\subsection{Qwen2.5-7B accuracy and judge evaluation}
\label{app:modern_7b}
Table~\ref{tab:7b_accuracy} reports per-seed reproduction and task accuracy
on Qwen2.5-7B with C-M, $r=1$, and $\eta=0.1$. Decoding is greedy, so seeds index
training and calibration runs rather than sampling.

\begin{table}[htbp]\centering\small
\caption{Qwen2.5-7B on GSM8K: reproduction and accuracy per seed ($r=1$). Rows with -- are the unmodified model.}
\label{tab:7b_accuracy}
\begin{tabular}{lrrr}
\toprule
Seed & $\beta$ & Reproduction & Accuracy \\
\midrule
1 & -- & 0.884 & 0.593 \\
1 & 1.0 & 0.833 & 0.597 \\
1 & 3.0 & 0.513 & 0.543 \\
1 & 10.0 & 0.246 & 0.460 \\
2 & -- & 0.854 & 0.597 \\
2 & 1.0 & 0.793 & 0.543 \\
2 & 3.0 & 0.416 & 0.467 \\
2 & 10.0 & 0.138 & 0.350 \\
44 & -- & 0.839 & 0.587 \\
44 & 1.0 & 0.829 & 0.573 \\
44 & 10.0 & 0.336 & 0.473 \\
\bottomrule\end{tabular}\end{table}

Two judges score 150 responses per condition on a 1--10 scale: a Qwen judge,
which shares the target's model family, and a Llama judge, which does not. We
test equivalence with TOST at a $\pm0.5$ margin; judge equivalence does not
imply equal task accuracy. At $\beta=3$, the Qwen judge supports equivalence for
GSM8K seed 1 and UltraChat but not for GSM8K seed 2, and stronger interventions
incur larger quality costs. Table~\ref{tab:judge2} lists the Qwen-judge conditions.

\begin{table}[htbp]\centering\small
\caption{Qwen2.5-7B-Instruct judge scores (C-M, $r=1$). Change is relative to the unmodified model; intervals and TOST $p$-values come from paired tests.}
\label{tab:judge2}
\begin{tabular}{lrrrrr}
\toprule
Setting & $\beta$ & Score & Change & 95\% CI & $p_{\mathrm{TOST}}$ \\
\midrule
GSM8K (seed 1) & 1.0 & 7.94 & -0.23 & [-0.52, 0.05] & 0.0361 \\
GSM8K (seed 1) & 3.0 & 8.02 & -0.15 & [-0.49, 0.17] & 0.0203 \\
GSM8K (seed 1) & 10.0 & 7.07 & -1.11 & [-1.56, -0.65] & 0.9958 \\
GSM8K (seed 2) & 1.0 & 7.71 & -0.25 & [-0.58, 0.09] & 0.0758 \\
GSM8K (seed 2) & 3.0 & 7.23 & -0.73 & [-1.13, -0.35] & 0.8703 \\
GSM8K (seed 2) & 10.0 & 6.47 & -1.49 & [-1.91, -1.05] & 1.0000 \\
UltraChat & 3.0 & 6.55 & -0.23 & [-0.47, 0.00] & 0.0140 \\
UltraChat & 10.0 & 5.87 & -0.91 & [-1.19, -0.63] & 0.9973 \\
\bottomrule\end{tabular}\end{table}

For GSM8K seed 1, the Llama judge reports mean change $+0.12$ at
$\beta=3$ and $-0.413$ at $\beta=10$, with TOST $p$ values
$7.1\times10^{-7}$ and 0.244. Judge-rated quality and final-answer accuracy
can therefore disagree, and neither indicates zero utility cost.

\subsection{Real-document domains}
\label{app:domain_results}
Qwen2.5-7B is LoRA~\cite{hu2022lora} fine-tuned on 150 documents per domain, duplicated eight
times for six epochs. The corpora are PubMedQA contexts, financial-news
articles, and ML-arXiv abstracts. Reproduction checks the first 50 continuation
tokens after a 50-token prefix under greedy decoding.
% MUST_VERIFY [DOMAIN-METRIC]: The new source consistently specifies tokens;
% the rebuttal says characters. Check the evaluation code before submission.
% Preserve the current table values (including PPL 2.302 and 2.011) unless logs
% show an error; rebuttal PPL values 2.10 and 1.91 are not silently substituted.
Table~\ref{tab:domain_grid} reports the full grid for C-M (rank one,
layer 21, $\epsilon$ the $(1-\eta)$ quantile of probe responses on the control
documents) with $r=1$ and $r=2$ over $\eta$ and $\beta$; the two gates share
the adapter, directions, threshold, and $\beta$ of each cell, and no single operating point is a universal default.
Because the clean control documents are also used to build the directions,
control PPL is not a fully independent test measure.

\begin{table}[htbp]\centering\footnotesize\setlength{\tabcolsep}{4pt}
\caption{Real-document domains, C-M with $r=1$ and $r=2$ (same directions, $\epsilon$, and $\beta$ per row). Exact-50 is exact match of the first 50 continuation tokens, PPL is measured on control documents, and trigger is the fraction of positions at which the gate corrects. Rows with -- are the unmodified model.}
\label{tab:domain_grid}
\begin{tabular}{lrrrrrrrr}
\toprule
 & & & \multicolumn{2}{c}{Exact-50} & \multicolumn{2}{c}{Control PPL} & \multicolumn{2}{c}{Trigger} \\
\cmidrule(lr){4-5}\cmidrule(lr){6-7}\cmidrule(lr){8-9}
Domain & $\eta$ & $\beta$ & $r=1$ & $r=2$ & $r=1$ & $r=2$ & $r=1$ & $r=2$ \\
\midrule
healthcare & -- & -- & \multicolumn{2}{c}{0.913} & \multicolumn{2}{c}{1.089} & \multicolumn{2}{c}{--} \\
healthcare & 0.05 & 1 & 0.233 & 0.840 & 2.019 & 1.138 & 0.049 & 0.004 \\
healthcare & 0.05 & 3 & 0.153 & 0.820 & 3.213 & 1.182 & 0.049 & 0.004 \\
healthcare & 0.05 & 10 & 0.133 & 0.813 & 4.332 & 1.215 & 0.050 & 0.004 \\
healthcare & 0.1 & 1 & 0.053 & 0.653 & 4.356 & 1.344 & 0.097 & 0.015 \\
healthcare & 0.1 & 3 & 0.040 & 0.613 & 12.173 & 1.611 & 0.097 & 0.015 \\
healthcare & 0.1 & 10 & 0.027 & 0.600 & 20.885 & 1.768 & 0.098 & 0.015 \\
finance & -- & -- & \multicolumn{2}{c}{0.933} & \multicolumn{2}{c}{1.098} & \multicolumn{2}{c}{--} \\
finance & 0.05 & 1 & 0.221 & 0.785 & 2.302 & 1.188 & 0.049 & 0.005 \\
finance & 0.05 & 3 & 0.168 & 0.752 & 3.655 & 1.228 & 0.049 & 0.005 \\
finance & 0.05 & 10 & 0.128 & 0.738 & 4.942 & 1.260 & 0.052 & 0.005 \\
finance & 0.1 & 1 & 0.074 & 0.617 & 5.303 & 1.425 & 0.095 & 0.016 \\
finance & 0.1 & 3 & 0.060 & 0.591 & 13.017 & 1.626 & 0.101 & 0.017 \\
finance & 0.1 & 10 & 0.054 & 0.564 & 22.502 & 1.776 & 0.114 & 0.019 \\
academia & -- & -- & \multicolumn{2}{c}{0.919} & \multicolumn{2}{c}{1.464} & \multicolumn{2}{c}{--} \\
academia & 0.05 & 1 & 0.919 & 0.933 & 1.473 & 1.468 & 0.049 & 0.011 \\
academia & 0.05 & 3 & 0.591 & 0.846 & 1.594 & 1.496 & 0.050 & 0.011 \\
academia & 0.05 & 10 & 0.302 & 0.685 & 2.011 & 1.567 & 0.050 & 0.011 \\
academia & 0.1 & 1 & 0.886 & 0.906 & 1.485 & 1.472 & 0.097 & 0.029 \\
academia & 0.1 & 3 & 0.362 & 0.685 & 1.833 & 1.561 & 0.100 & 0.029 \\
academia & 0.1 & 10 & 0.094 & 0.477 & 3.006 & 1.814 & 0.100 & 0.029 \\
\bottomrule\end{tabular}\end{table}

For example, at $(\eta,\beta)=(0.05,1)$, finance reproduction drops from 0.933
to 0.221 with PPL 2.302, and at $(0.05,10)$, academia reproduction drops from
0.919 to 0.302 with PPL 2.011 ($r=1$).
Requiring two consecutive exceedances ($r=2$) lowers the trigger rate by a
factor of 3.3--12.8 and keeps control PPL between 1.14 and 1.81, but it retains
only part of the suppression. At $(\eta,\beta)=(0.1,10)$, reproduction is
0.600, 0.564, and 0.477 with $r=2$ (healthcare, finance, academia), compared
with 0.027, 0.054, and 0.094 with $r=1$.

\subsection{Synthetic PII: duplication and extraction templates}
\label{app:pii_results}
Synthetic account records contain names with associated contact and account
fields; this evaluation is separate from the real-document domains. One hundred
fictional records (name, email on a reserved example domain, fictional phone
number, account ID) are mixed, with 1, 3, or 8 copies each, into the 1,000-item
GSM8K SFT set of Qwen2.5-7B (LoRA, six epochs). Extraction prompts each record
with its training request and decodes 80 tokens greedily; a record counts as
fully extracted when its email, phone number, and account ID all appear, and
email extraction checks the email alone. Accuracy is GSM8K final-answer
accuracy on the first 300 test items. Directions are rank-one C-M at layer 21,
accumulated over answer tokens, and $\epsilon$ is the 90th percentile
($\eta=0.1$) of probe responses on the answer tokens of the first 400 GSM8K test
items. Thus, the reported utility set overlaps with threshold-calibration data;
this is not a fully independent evaluation of the calibrated operating point.
Table~\ref{tab:pii_duplication} reports $r=1$ and $r=2$ with the same
adapter, directions, $\epsilon$, and $\beta$, including weak suppression under
heavy duplication and utility loss under strong steering. These canary results
do not constitute a privacy guarantee.

\begin{table}[htbp]\centering\small
\caption{Synthetic PII extraction by number of copies (C-M, $\eta=0.1$) with $r=1$ and $r=2$. Full-record and email columns are extraction rates over 100 records; accuracy is GSM8K test accuracy; trigger is the fraction of real (non-padding) token positions at which the gate corrects. Rows with -- are the unmodified model.}
\label{tab:pii_duplication}
\begin{tabular}{lrrrrrr}
\toprule
Copies & $\beta$ & $r$ & Full record & Email & Accuracy & Trigger \\
\midrule
1 & -- & -- & 0.00 & 0.02 & 0.597 & -- \\
1 & 3.0 & 1 & 0.00 & 0.03 & 0.337 & 0.069 \\
1 & 3.0 & 2 & 0.00 & 0.02 & 0.470 & 0.016 \\
1 & 10.0 & 1 & 0.00 & 0.02 & 0.340 & 0.062 \\
1 & 10.0 & 2 & 0.00 & 0.02 & 0.510 & 0.014 \\
\midrule
3 & -- & -- & 0.77 & 0.86 & 0.610 & -- \\
3 & 3.0 & 1 & 0.59 & 0.73 & 0.550 & 0.096 \\
3 & 3.0 & 2 & 0.76 & 0.83 & 0.567 & 0.024 \\
3 & 10.0 & 1 & 0.01 & 0.04 & 0.220 & 0.104 \\
3 & 10.0 & 2 & 0.35 & 0.38 & 0.467 & 0.031 \\
\midrule
8 & -- & -- & 1.00 & 1.00 & 0.600 & -- \\
8 & 3.0 & 1 & 1.00 & 1.00 & 0.587 & 0.067 \\
8 & 3.0 & 2 & 1.00 & 1.00 & 0.583 & 0.024 \\
8 & 10.0 & 1 & 0.84 & 1.00 & 0.473 & 0.065 \\
8 & 10.0 & 2 & 1.00 & 1.00 & 0.570 & 0.023 \\
\bottomrule\end{tabular}\end{table}

With $r=1$ and 8 copies at $\beta=10$, full-record extraction remains 0.84 and
email extraction 1.00; with 3 copies, strong suppression coincides with
accuracy falling from 0.61 to 0.22. Requiring two consecutive exceedances
($r=2$) lowers the trigger rate by a factor of 2.8--4.6 and recovers accuracy
(0.467 instead of 0.220 with 3 copies at $\beta=10$), but it suppresses much
less: with 3 copies, full-record and email extraction are 0.35 and 0.38 at
$\beta=10$ (0.01 and 0.04 with $r=1$) and 0.76 and 0.83 at $\beta=3$ (baseline
0.77 and 0.86), and with 8 copies full-record extraction stays at 1.00 at both
strengths.

We also vary the extraction prompt. Five templates request the record in
training format, a paraphrased lookup, all stored customer details, the email
alone, or the phone/account ID (100 records each, 3 copies). Across templates
that elicit the full record, baseline full-record extraction is 0.77--0.79 and
email extraction 0.86--0.87. With $r=1$, $\beta=3$ reduces these to 0.47--0.59
and 0.72--0.73, and $\beta=10$ to at most 0.01 and 0.04. With $r=2$, they are
0.74--0.78 and 0.80--0.84 at $\beta=3$, and 0.33--0.43 and 0.35--0.46 at
$\beta=10$. The email-only template does not ask for the full record, so its
full-record rate is zero even though the email itself is extracted (0.87 at
baseline; 0.76 and 0.05 with $r=1$ and 0.82 and 0.29 with $r=2$ at $\beta=3$
and 10).

\subsection{Approximate and semantic reproduction}
\label{app:approx_mem}
Zeng et al.~\cite{zeng2023exploring} study memorization during fine-tuning
and distinguish verbatim, paraphrase, and idea memorization. Idea memorization
can reveal information from a training continuation absent from the prompt
without reproducing its wording, for example by condensing that information.
Our evaluation below is an exploratory semantic-overlap analysis; it does not
reproduce their detection pipeline or establish mitigation under their metric.

The Pythia semantic sweep uses rank-one sign-canonicalized contrastive directions,
layer 22, nominal exceedance 0.1, and a single-position gate. Semantic similarity
is a MiniLM embedding-cosine proxy. We use ``semantic-only proxy'' for the
fraction of outputs with cosine greater than 0.8 without exact reproduction.
This proxy does not by itself identify exposure-caused information leakage or
distinguish it from legitimate semantic similarity.
% MUST_VERIFY [SEMANTIC]: Keep the current table pending log verification.
% For Exact=.248, the rebuttal gives 4-gram=.645 and idea-only=.208;
% the supplied source gives .545 and .356, reversing the proxy conclusion.
% Confirm embedding model, reference text, exclusions and denominator.

\begin{table}[htbp]\centering\small
\caption{Exact and semantic reproduction on Pythia-2.8B across steering strengths. Weak steering increases the semantic-only proxy, so exact suppression does not imply semantic suppression.}
\label{tab:semantic_sweep}
\begin{tabular}{lrrrr}
\toprule
$\beta$ & Exact & 4-gram precision & Cosine & \shortstack{Semantic-only\\proxy} \\
\midrule
Baseline & 0.535 & 0.811 & 0.913 & 0.307 \\
1.0 & 0.248 & 0.545 & 0.785 & 0.356 \\
3.0 & 0.141 & 0.425 & 0.690 & 0.299 \\
5.0 & 0.107 & 0.365 & 0.641 & 0.259 \\
15.0 & 0.081 & 0.307 & 0.592 & 0.218 \\
20.0 & 0.077 & 0.286 & 0.575 & 0.198 \\
30.0 & 0.046 & 0.256 & 0.545 & 0.198 \\
\bottomrule\end{tabular}\end{table}

At strength 1, exact reproduction decreases but the semantic-only proxy rises
from 0.307 to 0.356. Stronger correction reduces all overlap proxies, at a
potential cost in generation quality. A lower exact-match rate can therefore
coexist with more non-exact, semantically similar outputs. These results do not
support a general claim that GSS prevents idea memorization. Evaluating that
claim requires a dedicated exposure-sensitive protocol and an independent
utility assessment.

Tables~\ref{tab:approx_tinymem} and~\ref{tab:approx_pythia} use a separate
protocol with a similarity threshold of 0.5 for approximate reproduction.

\begin{table}[h]
\centering
\small
\caption{TinyMem exact and approximate reproduction ($r=1$).}
\label{tab:approx_tinymem}
\begin{tabular}{lccc}
\toprule
Metric & Before & After & Delta \\
\midrule
Exact memorized \% & 23.97 & 0.00 & $-23.97$ \\
Approx-only memorized \% & 76.03 & 3.64 & $-72.39$ \\
ROUGE-L mean & 0.9664 & 0.3749 & $-0.5916$ \\
\bottomrule
\end{tabular}
\end{table}

\begin{table}[h]
\centering
\small
\caption{Pythia-2.8B exact and approximate reproduction ($r=1$).}
\label{tab:approx_pythia}
\begin{tabular}{lccc}
\toprule
Metric & Before & After & Delta \\
\midrule
Exact memorized \% & 46.14 & 8.91 & $-37.23$ \\
Approx-only memorized \% & 42.18 & 34.06 & $-8.12$ \\
ROUGE-L mean & 0.8529 & 0.4740 & $-0.3790$ \\
\bottomrule
\end{tabular}
\end{table}

\subsection{Extraction attacks}
\label{app:adversarial}
Table~\ref{tab:adv_extraction} lists the attack settings in which steering
reduced attack success. It is a subset of the evaluated grid and does not
indicate uniform robustness across attacks.

\begin{table}[h]
\centering
\small
\caption{Attack settings in which steering reduces attack success rate (a subset of the evaluated grid).}
\label{tab:adv_extraction}
\begin{tabular}{lcccc}
\toprule
Attack template & Prefix fraction & Base ASR (\%) & Steered ASR (\%) & Reduction (\%) \\
\midrule
Instruction-0 & 0.2 & 9.7 & 5.1 & 47 \\
Instruction-0 & 0.5 & 18.5 & 9.2 & 50 \\
Instruction-2 & 0.3 & 16.4 & 5.1 & 69 \\
Paraphrase-0 & 0.3 & 16.4 & 5.1 & 69 \\
Paraphrase-0 & 0.5 & 18.5 & 7.2 & 61 \\
\bottomrule
\end{tabular}
\end{table}

\begin{table}[h]
\centering
\small
\caption{Best-of-8 extraction on 145 sequences ($r=1$). Exact-match ASR is already zero at baseline, so overlap-based ASR is the informative outcome.}
\label{tab:adv_bo8}
\begin{tabular}{llccrr}
\toprule
Template & Prefix Fraction & \# Examples & Metric & Base & Steered \\
\midrule
Raw prefix & 0.5 & 145 & Exact Match ASR & 0.00 & 0.00 \\
Raw prefix & 0.5 & 145 & ASR@50\% Token Overlap & 54.48 & 0.00 \\
Raw prefix & 0.5 & 145 & ASR@80\% Token Overlap & 2.07 & 0.00 \\
Raw prefix & 0.5 & 145 & Mean Token Overlap & 0.4469 & 0.1390 \\
Raw prefix & 0.5 & 145 & Mean LCP Ratio & 0.2855 & 0.0172 \\
Raw prefix & 0.5 & 145 & Mean ROUGE-L & 0.6404 & 0.3792 \\
\bottomrule
\end{tabular}
\end{table}

\paragraph{Interpreting zero successes.}
Under an independent sequence-sampling model, zero successes among the 145
sequences in Table~\ref{tab:adv_bo8} gives a two-sided 95\% exact binomial
upper limit of $1-0.025^{1/145}\approx2.51\%$ for success within the specified
attack budget. The eight attempts on each sequence are not 1,160 independent
sequence trials. This conditional interval is not a bound for arbitrary
prompts, adaptive attacks, or another sequence population.

\FloatBarrier
\section{Further ablations and additional results}
\subsection{Probe discrimination and stability}
\label{app:probe}
Table~\ref{tab:probe_discrim} reports how well the probe subspace separates
memorized from non-memorized sequences, and Table~\ref{tab:probe_stable} its
stability under resampling. Stability alone does not establish a causal
memorization subspace or transfer across domains.

\begin{table}[h]
\centering
\small
\caption{Sequence-level discrimination performance of the whitened probe subspace on memorized vs.\ non-memorized sequences.}
\label{tab:probe_discrim}
\begin{tabular}{llcc}
\toprule
Model & Metric & Rank 4 & Rank 8 \\
\midrule
Qwen3-0.6B & AUC & 0.783 & 0.811 \\
Qwen3-0.6B & Spearman & 0.530 & 0.603 \\
Llama-3.2-1B & AUC & 0.899 & 0.929 \\
Llama-3.2-1B & Spearman & 0.771 & 0.839 \\
\bottomrule
\end{tabular}
\end{table}

\begin{table}[h]
\centering
\small
\caption{Stability of the learned probe and steering directions across resampled training subset fractions.}
\label{tab:probe_stable}
\begin{tabular}{lc}
\toprule
Direction & $\cos$ \\
\midrule
$u$ (probe) & $0.999 \pm 0.001$ \\
$v$ (steer) & $0.943 \pm 0.067$ \\
\bottomrule
\end{tabular}
\end{table}

\subsection{Long-form generation}
\label{app:generation}
We compare 400-token greedy outputs of Llama-3.2-1B before and after
intervention. ROUGE-L between the two outputs measures their similarity, not
correctness, and Distinct-$n$ measures lexical diversity.

\begin{table}[h]
\centering
\small
\caption{Long-form generation on Llama-3.2-1B: ROUGE-L between baseline and GSS outputs, and change in Distinct-$n$.}
\label{tab:generation}
\begin{tabular}{lccc}
\toprule
Condition & ROUGE-L (Base, GSS) & Distinct-2 $\Delta$ & Distinct-3 $\Delta$ \\
\midrule
Clean (unseen test) & 0.835 & +0.009 & +0.004 \\
Memorized (train prefix) & 0.685 & +0.023 & +0.031 \\
\bottomrule
\end{tabular}
\end{table}

\subsection{Layer selection}
\label{app:layer_ablation}
Layer-dependent probe statistics can guide calibration choices, but the best
layer for a signal statistic need not optimize the suppression--utility trade-off.
Figure~\ref{fig:ablation_layer} shows the per-layer probe statistics.

\begin{figure}[htbp]\centering
\IfFileExists{fig/ablation/layer_probe_ratios_data.png}{\includegraphics[width=0.55\linewidth]{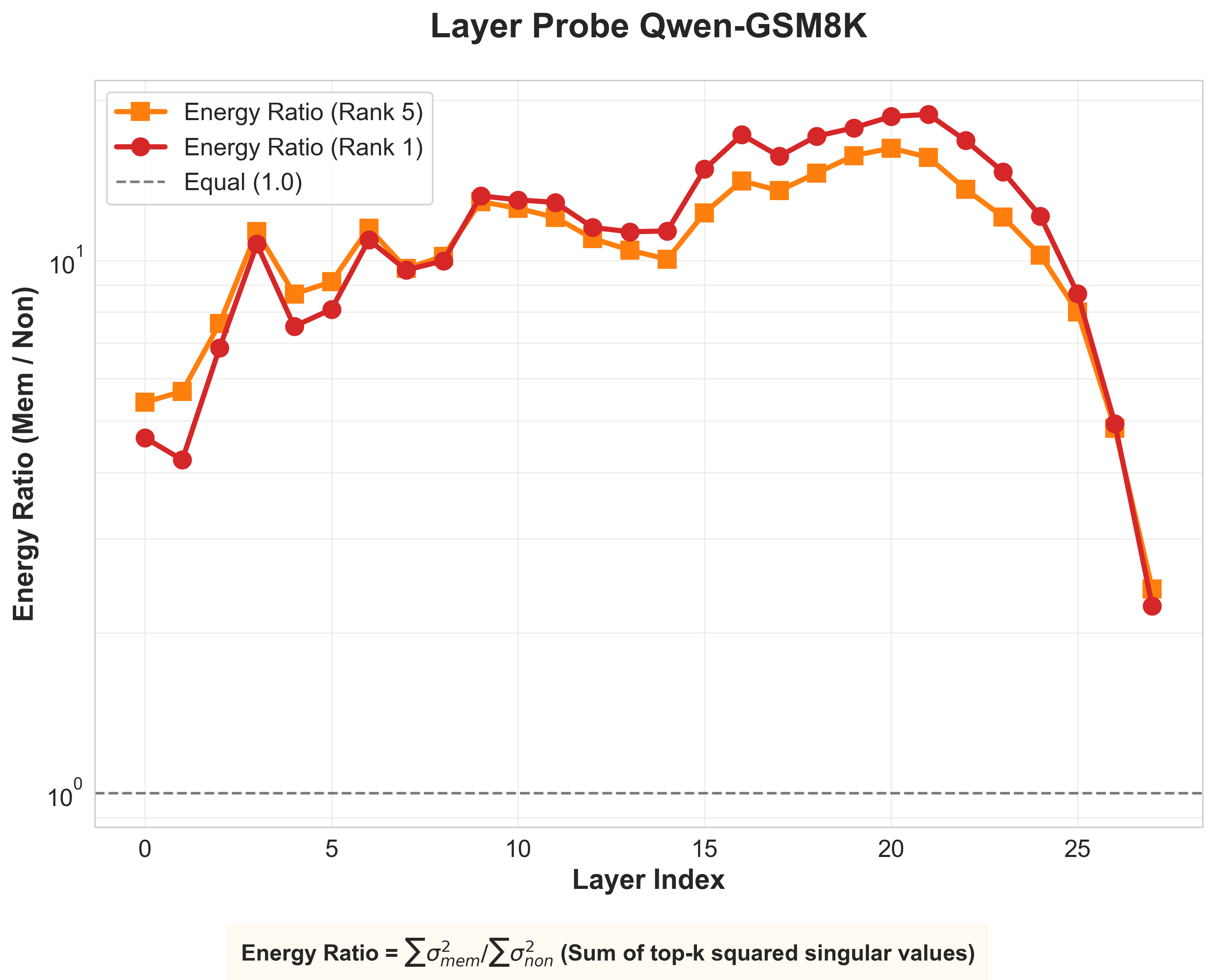}}{\fbox{\parbox{0.85\linewidth}{Original figure asset required: \texttt{\detokenize{fig/ablation/layer_probe_ratios_data.png}}.}}}
\caption{Per-layer probe statistics on Qwen3-0.6B, GSM8K.}
\label{fig:ablation_layer}
\end{figure}

\subsection{Bootstrap stability of activation subspaces}
\label{app:subspace-similarity}
For two orthonormal bases $U_i,U_j\in\mathbb R^{d\times K}$, define
$S_{ij}=\|U_i^\top U_j\|_*/K\in[0,1]$.
We compare bootstrap PCA subspaces estimated on
positive-signal positions, nonpositive-signal controls, and random Gaussian
subspaces. These PCA subspaces differ from the gradient-weighted GSS
directions. Higher agreement across resamples indicates stability in this
setting; it does not show that all memorization is low-dimensional.

\begin{figure}[htbp]\centering
\IfFileExists{fig/subspace_similarity_appendix.png}{\includegraphics[width=0.88\linewidth]{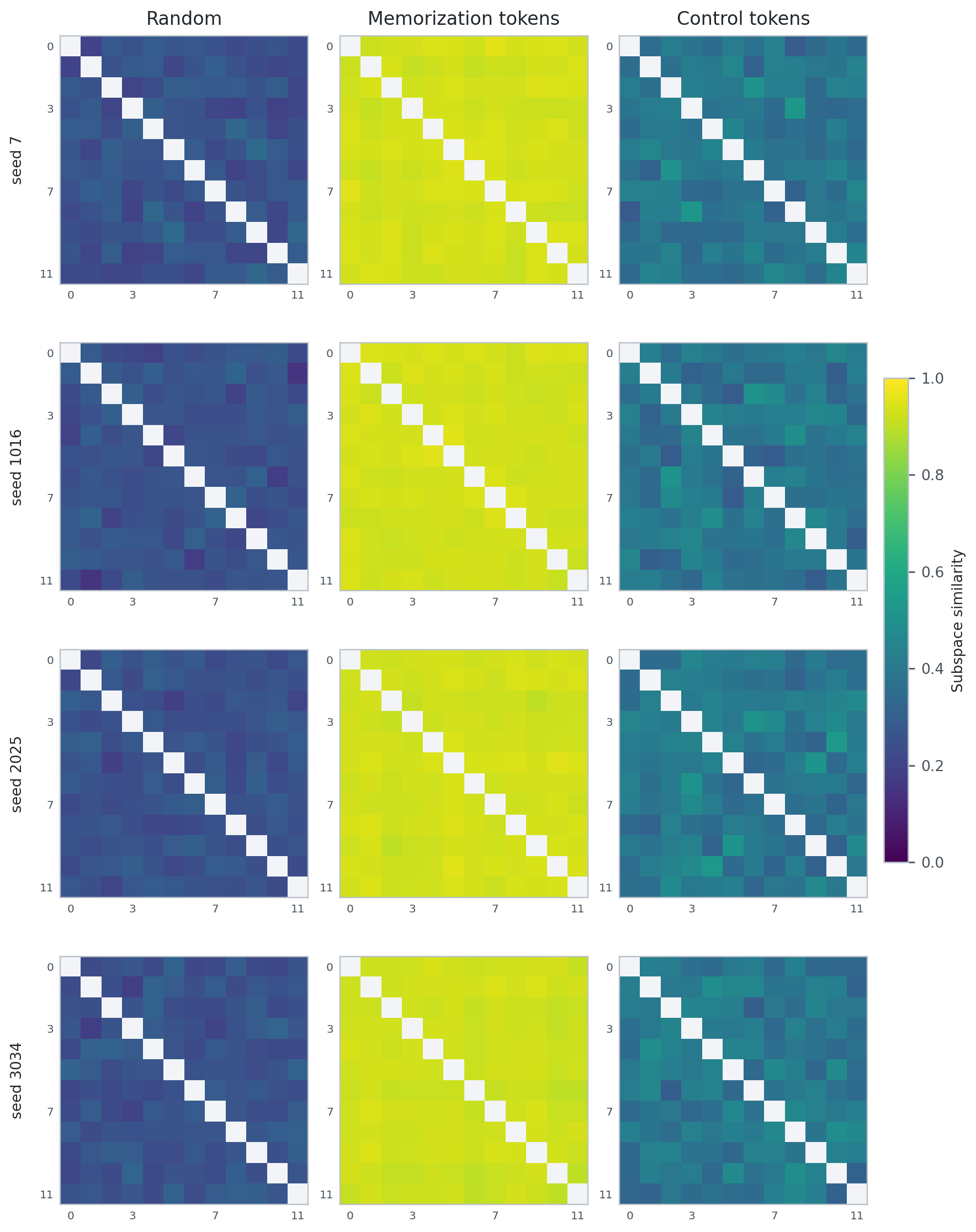}}{\fbox{\parbox{0.85\linewidth}{Original figure asset required: \texttt{\detokenize{fig/subspace_similarity_appendix.png}}.}}}
\caption{Bootstrap similarity of PCA activation subspaces.}
\label{fig:appendix-subspace-similarity}
\end{figure}

\subsection{Utility-constrained modern-model results}
\label{app:llama_qwen}
Table~\ref{tab:categorized_results} reports, for each method, the largest
reproduction reduction among operating points whose normalized utility is at
least 0.9. GSS gives the largest reduction on GSM8K for both models, while
several unlearning baselines achieve larger reductions on UltraChat.

\begin{table}[h]
\centering
\small
\setlength{\tabcolsep}{4pt}
\caption{Largest reproduction reduction (\%) subject to normalized utility of at least 0.9.}
\begin{tabular}{c|l|cc|cc}
\toprule
\multirow{2}{*}{\textbf{Category}} & \multirow{2}{*}{\textbf{Method}} &
\multicolumn{2}{c|}{\textbf{Llama (Red.~\%)}} &
\multicolumn{2}{c}{\textbf{Qwen (Red.~\%)}} \\ \cmidrule{3-6}
 & & \textbf{UltraChat} & \textbf{GSM8K} & \textbf{UltraChat} & \textbf{GSM8K} \\
\midrule
\multirow{5}{*}{\rotatebox{90}{\textbf{Unlearning}}}
 & GA & 95.5 & 4.5 & 3.3 & 0.0 \\
 & Task Vector & 77.0 & 25.1 & 66.0 & 25.4 \\
 & DPO & 95.7 & 1.3 & 82.4 & 7.3 \\
 & SimNPO & 95.5 & 2.2 & 78.8 & 7.2 \\
 & NPO & 88.8 & 2.2 & 71.5 & 6.9 \\
\midrule
\multirow{4}{*}{\rotatebox{90}{\textbf{Editing}}}
 & ROME & 38.0 & 7.8 & 8.6 & 14.5 \\
 & MEMIT & 7.5 & 2.2 & 2.2 & 8.7 \\
 & Balanced Subnet & 34.8 & 1.4 & 1.0 & 12.8 \\
 & LUNAR & 30.0 & 9.0 & 42.7 & 1.0 \\
\midrule
 & \textbf{Ours} & 48.0 & 33.5 & 51.4 & 35.2 \\
\bottomrule
\end{tabular}
\label{tab:categorized_results}
\end{table}

\subsection{Qualitative examples}
\label{app:qualitative}
Figure~\ref{fig:example} shows contract-QA outputs before and after steering.
Repetition, malformed strings, or unusual numbers do not by themselves indicate
memorization, which requires a match to a training sequence.

\begin{figure}[htbp]\centering
\IfFileExists{fig/appendix/qualitative.pdf}{\includegraphics[width=0.88\linewidth]{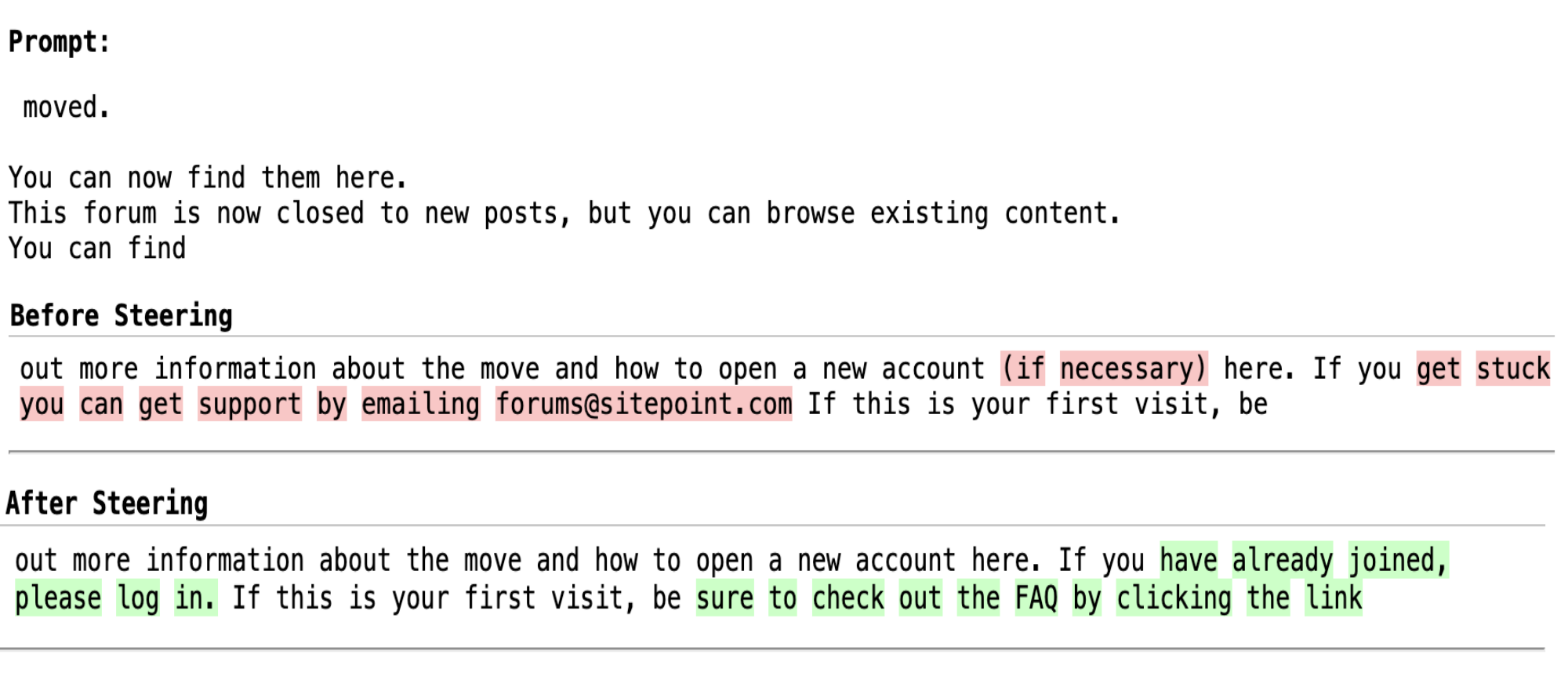}}{\fbox{\parbox{0.85\linewidth}{Original figure asset required: \texttt{\detokenize{fig/appendix/qualitative.pdf}}.}}}
\caption{Contract-QA outputs before and after steering.}
\label{fig:example}
\end{figure}

\FloatBarrier
\section{Methodological connections}
\label{app:Connections}
\subsection{Low-rank updates}
Ignoring gating, a linear activation update $h'=(I-\beta VU^\top)h$
placed immediately before a linear map $W$ gives
$Wh'=(W-\beta WVU^\top)h$. This is algebraically a rank-at-most-$K$
weight modification at that location. It does not imply equivalence to a
trained LoRA adapter throughout a nonlinear Transformer. A gate makes the
update input-dependent, and a ReLU excess coefficient also includes a
threshold-dependent term. Reversibility means the intervention can be disabled;
it does not certify removal of stored information from the unchanged weights.

\subsection{Relation to discriminant analysis}
\label{app:fda}
Classical discriminant analysis uses class-scatter statistics. Here the
coupling matrix combines activations and log-likelihood gradients
(negated CE gradients; Appendix~\ref{app:calibration_gradients}), while the whitened
configuration normalizes an uncentered low-signal second moment.
This asymmetric coupling motivates separate probe and steer directions, and
setting $u=v$ does not in general recover mean-difference steering.

\subsection{Unlearning and editing}
Machine unlearning studies removal of the influence of designated training
examples; surveys and methods examine this objective across learning
settings~\cite{nguyen2025survey,yao2024machine,chen2023fast,tarun2023fast,xu2025unlearning,huo2025mmunlearner,wang2025selective}.
Knowledge editing instead studies targeted changes to model knowledge or
behavior~\cite{hu2024wilke,jiang2024learning,jiang2025anyedit}.
Unlearning and editing provide objective-related comparisons for reducing
specified reproduction behavior. Activation steering provides mechanism-related
comparisons. GSS can reduce reproduction while active without changing the
underlying weights; this is not exact unlearning or a retraining-equivalence
guarantee.

\FloatBarrier
\section{Broader impact and limitations}
\label{app:impact}

\subsection{Societal implications}
\label{app:societal}

Reducing verbatim reproduction may lower exposure to training-data extraction
and unauthorized reproduction of sensitive or copyrighted content under the
evaluated protocols. GSS does not remove information from the model weights,
and these experiments do not establish comprehensive privacy protection,
fairness improvements, or out-of-distribution robustness.

We acknowledge potential dual-use concerns: techniques mitigating memorization could obscure training data sources, complicating model auditing; selective memorization reduction could be misused to hide specific training data while maintaining overall performance; and over-aggressive mitigation might remove useful factual knowledge.

\subsection{Limitations}

Beyond the proxy-related limitations discussed in Appendix~\ref{app:proxy_validity}, our approach has several practical limitations:
\begin{enumerate}
 \item \textbf{Reference-model dependence}: Computing $\omega_t$ requires a domain-competent, lower-exposure reference.
The original pretrained checkpoint is not required, but a suitable alternative
may be unavailable; arbitrary model disagreement is not sufficient.
 \item \textbf{Semantic leakage}: Exact-match suppression does not guarantee
mitigation of paraphrase or idea memorization (Appendix~\ref{app:approx_mem}).
 \item \textbf{Calibration and utility}: Domain control PPL and the PII utility
set are not fully independent of calibration. Results should not be interpreted
as independently validated deployment guarantees.
 \item \textbf{Calibration cost}: Direction estimation requires activations,
gradients and a matrix decomposition. Online application cost excludes these
one-time computations and must be reported separately.

 \item \textbf{Layer selection}: Effectiveness depends on choosing appropriate layers, which may require empirical validation.
 \item \textbf{Run requirement}: Requiring two consecutive exceedances ($r=2$) lowers the trigger rate in every evaluated setting, and the share of suppression it keeps depends on the setting. On Pythia-2.8B, $r=2$ keeps 49.0 of the 52.3 points of verbatim reduction of $r=1$ (Table~\ref{tab:consecutive_pythia_cohort}). On fine-tuned Qwen2.5-7B, it keeps 14.8 of 45.8 points on GSM8K (Table~\ref{tab:consecutive_7b_cohort}); at $(\eta,\beta)=(0.1,10)$, reproduction on the three real-document domains is 0.600, 0.564, and 0.477 with $r=2$ against 0.027, 0.054, and 0.094 with $r=1$ (Table~\ref{tab:domain_grid}); and with eight PII copies at $\beta=10$, full-record extraction is 1.00 with $r=2$ against 0.84 with $r=1$ (Table~\ref{tab:pii_duplication}). In these fine-tuning settings, $r$ is a parameter that trades suppression for utility and is chosen per setting.
\end{enumerate}

\subsection{Future directions}

Our results suggest several promising directions for future work. 
First, the current study focuses on text-only language models, while memorization is also a central concern in vision-language and other multi-modal systems~\cite{wen2025quantifying}. 
Extending GSS to multi-modal models would require identifying whether memorization-relevant activations are shared across modalities or localized within modality-specific encoders, such as visual tokenizers, audio encoders, or cross-modal projection layers. 
This would also allow us to study whether the concentration of positive likelihood signal observed in language modeling persists when the model conditions on images, speech, or structured visual inputs.

Additionally, memorization behavior may differ substantially across languages and writing systems~\cite{satvaty2025memorization}. 
Multilingual models are trained on highly imbalanced corpora, and low-resource languages may exhibit different memorization patterns due to duplicated data, translation artifacts, or uneven tokenization quality. 
Extending our analysis to multilingual settings would test whether a single shared memorization subspace is sufficient, or whether language-specific probes and steering directions are needed. 
Together, these directions would broaden GSS from a text-only mitigation method into a more general framework for context-aware memorization control.
\FloatBarrier
\stopcontents[appendices]

%\newpage
%\input{appendix/checklist.tex}

\end{document}